\pdfoutput=1
\documentclass[11pt]{article}
\usepackage{fullpage,times}
\usepackage{parskip}

\usepackage{amsmath, amsthm, amssymb, amsfonts, mathtools, graphicx, enumerate}
\usepackage{adjustbox}
\usepackage[utf8]{inputenc} 
\usepackage[T1]{fontenc}    
\usepackage{url}            
\usepackage{booktabs}
\usepackage{nicefrac}       
\usepackage{microtype}      
\usepackage{xspace}
\usepackage{algorithm,algorithmic}
\usepackage{color}
\usepackage{enumitem}
\usepackage{comment}
\usepackage{bm}
\usepackage{apptools}
\usepackage[page, header]{appendix}
\usepackage{titletoc}
\usepackage{array}

\newcount\Comments  %
\newcommand{\kibitz}[2]{\ifnum\Comments=1\textcolor{#1}{#2}\fi}

\usepackage[scaled=.9]{helvet}

\usepackage{url}
\usepackage{caption}
\usepackage[table,dvipsnames]{xcolor}
\usepackage[colorlinks=true, linkcolor=blue, citecolor=blue]{hyperref}
\usepackage{qiangstyle}

\usepackage{natbib}
\bibpunct{(}{)}{;}{a}{,}{,}

\usepackage{lipsum}  
\makeatletter
\newcommand{\printfnsymbol}[1]{%
  \textsuperscript{\@fnsymbol{#1}}%
}

\title{Non-asymptotic Confidence Intervals of Off-policy Evaluation: 
Primal and Dual Bounds }
\date{}

\author{%
	Yihao Feng\thanks{Equal contribution. This article is an extended version of \citet{feng2021nonasymptotic} in \textit{ICLR 2021}. }\\
    ~University of Texas at Austin\\
	\texttt{yihao@cs.utexas.edu}
	\and
	Ziyang Tang\printfnsymbol{1}\\
	~University of Texas at Austin\\
	\texttt{ztang@cs.utexas.edu} \\
	\and
	Na Zhang \\
	Tsinghua University\\
	\texttt{zhangna@pbcsf.tsinghua.edu.cn} 
	\and 
	Qiang Liu \\
	~University of Texas at Austin\\
	\texttt{lqiang@cs.utexas.edu} \\
	
}

\begin{document}
\maketitle

\begin{abstract}
Off-policy evaluation (OPE) is the task of estimating the expected reward of a given policy based on offline data previously collected under different policies. 
Therefore, OPE is a key step in applying reinforcement learning to real-world domains such as medical treatment, where interactive data collection is expensive or even unsafe. 
As the observed data tends to be noisy and limited, 
it is essential to provide rigorous  uncertainty quantification, 
not just a point estimation, when applying OPE to make high stakes decisions. 
This work considers the problem of constructing non-asymptotic confidence intervals in infinite-horizon  off-policy evaluation, which remains a challenging open question. 
We develop a practical algorithm 
through a primal-dual optimization-based approach, %
which leverages the kernel Bellman loss (KBL) of \citet{feng2019kernel}  %
and a new  martingale concentration inequality of KBL applicable to time-dependent data with unknown mixing conditions. 
Our algorithm makes  minimum assumptions on the data and the function class of the Q-function, 
and works for the behavior-agnostic settings where the data is %
collected under a mix of arbitrary unknown behavior policies. 
We present empirical results that clearly demonstrate the advantages of our approach over existing methods.

\end{abstract}

\section{Introduction}

Off-policy evaluation (OPE) seeks to estimate the expected reward of a target policy in reinforcement learning (RL) from observational data collected under different behavior policies \citep[e.g.,][]{murphy01marginal,fonteneau13batch, jiang16doubly, liu2018breaking}. 
OPE plays a central role in applying reinforcement learning (RL) with only observational data and has found important applications in areas such as medical treatments, autonomous driving, where interactive ``on-policy'' data is expensive or even infeasible to collect.

As the observed data is often noisy and limited, applying OPE to assist high-stakes decision-making entails a rigorous, non-asymptotic approach to quantify the uncertainty in the estimation. In this way, the decision makers are informed of the degree of the risk
and can avoid the dangerous case of being overconfident in making costly and/or irreversible decisions. 
This work addresses the uncertainty quantification in OPE by constructing \emph{non-asymptotically correct confidence intervals} that contains the true expected reward of the target policy with a probability no smaller than a user-specified confidence level.

To date, Off-policy evaluation per se has remained a key technical challenge in the literature
\citep[e.g.,][]{precup2000eligibility,thomas16data,jiang16doubly, liu2018breaking}, 
let alone gaining rigorous confidence estimation of it. 
This is especially true when 1) the underlying RL problem is long or infinite horizon (a.k.a. \emph{the curse of horizon} \citep{liu2018breaking}); and 2) the data is collected under arbitrary and unknown algorithms (a.k.a. behavior-agnostic).
In such cases, the collected data can exhibit complex and unknown time-dependency structure, which makes constructing rigorous non-asymptotic confidence bounds particularly challenging.  
Traditionally, the main approach of deriving non-asymptotic confidence bounds in OPE is to combine importance sampling (IS) with certain concentration inequality \citep[e.g.,][]{thomas15high}. However, the classical IS-based methods tend to degenerate for long or infinite horizon problems \citep{liu2018breaking,liu2019understanding}. %
Moreover, they cannot be applied to the behavior-agnostic settings and also fail to handle the complicated time-dependency structure inside individual trajectories. In fact, the IS-based methods only work for short-horizon problems with a large number of independently collected trajectories drawn under known policies. %

In this work, we provide 
a practical algorithm with theoretical guarantee 
for 
\textbf{B}ehavior-agnostic, 
\textbf{O}ff-policy, 
\textbf{I}nfinite-horizon, 
\textbf{N}on-asymptotic, 
\textbf{C}onfidence intervals based on arbitrarily \textbf{D}ependent data  
(\textbf{BONDIC}). 
This method is motivated by  
a recently proposed optimization-based 
approach to estimating OPE confidence bounds \citep{fengaccountable2020}, which leverages a tail bound of kernel Bellman loss \citep{feng2019kernel}. 
Our approach achieves a new dual bound that is both an order-of-magnitude tighter and more computationally  efficient than 
that of \citet{fengaccountable2020}.
These improvements are based on two pillars:   
1)  a new primal-dual perspective on the non-asymptotic confidence bounds of infinite-horizon OPE; %
and 2) a new martingale concentration inequality on the kernel Bellman loss that applies to behavior-agnostic off-policy data with arbitrary time-dependency between transition pairs.
Our simple and practical method can provide 
reliable and informative bounds on a variety of 
benchmarks.

\myparagraph{Outline}
The rest of the paper is organized as follows. We introduce the problem setting in Section~\ref{sec:problem} and discuss related works  in Section~\ref{sec:relatedworks}. %
In Section~\ref{sec:infinite}, we give an overview on two dual approaches to infinite-horizon OPE that are closely related to our method (but do not consider non-asymptotic error bounds). 
We then present a martingale concentration bound for kernel Bellman loss in Section~\ref{sec:concentration}, which is used in our main approach described in Section~\ref{sec:main}.
Finally, we perform empirical studies in Section~\ref{sec:experiments}. 
The proofs and additional discussions can be found in Appendix.

\section{Background, Data Assumption, Problem Setting}
\label{sec:problem}

Consider an agent acting in an unknown environment. 
At each time step $t$, the agent observes the current state $s_t$ in a state space $\set S$, takes an action $a_{t} \sim \pi(\cdot ~|~s_t)$ in an action space $\set A$ according to a \emph{given} policy $\pi$; then, the agent receives a reward $r_t$ and the state transits to $s_t' = s_{t+1}$, following an \emph{unknown} transition/reward distribution $(r_t, s_{t+1})\sim  \dist P(\cdot ~|~ s_t, a_t)$. Assume the initial state $s_0 $ is drawn from an \emph{known} initial distribution $\dist d_0$.   
Let $\gamma\in(0,1)$ be a discount factor. In this setting, the expected reward of $\pi$ is defined as 
 \begin{align}\label{equ:jpi0}
 J\true\defeq J_{\pi, \dist P} \defeq  \lim_{H\to+\infty} \E_{\pi, \dist P}\left [\sum_{t=0}^{H}\gamma^{t}r_t ~|~ s_0\sim \dist d_0 \right ], 
 \end{align}
 which is the expected total discounted rewards when we execute $\pi$ starting from $\dist d_0$ for $H$ steps. 
 The focus of this work is 
the \emph{infinite-horizon} case with $H \to +\infty$, 
although our method can also be applied to a finite-horizon, time-inhomogeneous problem by converting it into an infinite-horizon and time-homogeneous problem that incorporates time $t$ as a special state variable (Appendix~\ref{sec:finitehorizon}).  
The state-action space $\set S \times \set A$ can be any domain on which we can define a reproducing kernel Hilbert space (RKHS), 
either discrete or continuous.
 
Assume the model $\dist P$ is unknown, but we can observe a set of transition pairs $\data = (s_i, a_i, r_i, s_i')_{i=1}^n$
that obeys $\dist P$ in that 
$(r_i, s_i')\sim \dist P(\cdot ~|~ s_i, a_i)$. %
The goal is to construct an interval $[\hat J^-(\emp D_n), ~ \hat J^+(\emp D_n)]$ based on the observed data such that the probability that it contains the true $J\true$ is no smaller than a user-specified confidence level. 
We consider the challenging case where the data is \emph{off-policy}, \emph{behavior-agnostic}, and \emph{time-dependent}, as formally specified in the following.

\begin{ass}[\textbf{Data Assumption}]  \label{ass:data}
Assume the data  $\data = (s_i, a_i, r_i, s_i')_{i=1}^n$ is drawn from an unknown joint distribution $\dist D_{1:n}^\diamond$ on $(\set S \times \set A \times \RR \times \set S)^{n}$
that satisfies 
\bba \label{equ:data}
\dist D_{1:n}^\diamond(r_i, s_i' ~|~ s_i, a_i; ~\emp D_{<i}) = \dist P(r_i, s_i' ~|~s_i, a_i), ~~~~~\forall i=1,\ldots, n,
\eea 
where $\emp D_{<i} \defeq (s_j,a_j, r_j, s_j')_{j<i}$.  
\end{ass} 

\myparagraph{Goal}  
\emph{ Given a dataset $\emp D_n$, a target policy $\pi$, and 
a confidence level $\delta\in(0,1)$, 
we want to construct  an 
interval $[\hatJlow(\emp D_n), \hatJup(\emp D_n)] \subset \RR$, %
such that 
\begin{align}\label{equ:probbound}
\prob%
\left ( \Jpi \in \left [\hatJlow(\emp D_n),~~ \hatJup(\emp D_n) \right]\right) \geq 1-\delta,
\end{align}
where $\prob(\cdot)$ is w.r.t. the 
randomness 
of the data $\emp D_n$ drawn from the distribution $\dist D_{1:n}^\diamond$. 
}

The data assumption provides a partial specification of the joint distribution $\dist D_{1:n}^\diamond$, which only requires that $(r_i, s'_i)$ should be generated from $\dist P(\cdot ~|~s_i,a_i)$ given $(s_i,a_i)\cup \emp D_{<i}$ at each step, while imposing no constraints on how $(s_i, a_i)$ is generated based on $\emp D_{<i}$.  
This provides great flexibility in terms of the data collection procedure. 
For example, $\emp D_n$ 
can be either collected from a single MDP %
governed by 
an \emph{unknown, time-varying, non-Markovian} behavior policy (in which case $(s_i',a_i')=(s_{i+1}, a_{i+1})$), or 
be a combination of many short MDP segments 
(in which case $(s_i',a_i')$ may not equal $(s_{i+1}, a_{i+1})$). 
Note that this significantly relaxes the data assumptions in recent works \citep{liu2018breaking, mousavi2020blackbox, dai2020coindice}, 
which require  $(s_i,a_i)_{i=1}^n$ to be independent or i.i.d. 

Eq~\eqref{equ:probbound} is a \emph{correctness} requirement of the confidence interval. 
In addition, we also want to make the interval estimation as \emph{tight} as possible. %
Specifically, it is desirable that  
 the length $(\hat J^+(\emp D_n) - \hat J^-(\emp D_n))$ of the interval  vanishes to zero as $n\to\infty$, 
 ideally with a fast converge rate (e.g., $O(n^{-1/2})$). 
 While assumption~\ref{ass:data} allows us to construct bounds that are correct and non-vacuous, it is too mild to give meaningful \emph{a priori} bound on the tightness of the interval estimation.  
For example, under Assumption~\ref{ass:data}, it is possible that the whole dataset is the replication of a single transition pair, i.e., $(s_i,a_i,r_i, s_i') = (s_1,a_1,r_1, s_1')$, in which case we cannot have a $1-\delta$ confidence interval whose length vanishes to zero as $n\to\infty$.   

A primary advantage of having an interval estimation is that 
we know \emph{a posteriori} length of the confidence interval once we construct it, 
which provides an instance-based uncertainty estimation 
without making any additional assumption beyond Assumption~\ref{ass:data}. 
If a stronger assumption is imposed (e.g., $(s_i,a_i,r_i, s_i')_{i=1}^n$ are i.i.d. or weakly dependent in a proper sense for different $i$), 
we can derive \emph{a priori} estimation of the convergence rate of the tightness $(\hat J^+(\emp D_n) - \hat J^-(\emp D_n))$ for the confidence interval as shown in Section~\ref{sec:dual} and Appendix~\ref{sec:tightness}.   

A crucial fact is that the data assumption above implies a martingale structure on the empirical Bellman residual operator of the Q-function. As we will show in Section~\ref{sec:concentration}, this enables us to derive a key concentration inequality underpinning our non-asymptotic confidence bound.

\begin{table}[t]
\centering
\renewcommand{\arraystretch}{1.5}
\setlength{\tabcolsep}{12pt}
\begin{tabular}{c|c|c|c|c|c} 
 \hline
 state &  action & reward &   state-action &next state-action &  next state-action + reward \\ \hline
 $s_i$ & $a_i$ & $r_i$ & $x_i = (s_i,a_i)$ & $x_i'=(s_i',a_i')$ & $y_i = (s_i',a_i',r_i)$\\ 
 \hline
\end{tabular}
\caption{Summary of  the notation. 
We assume we observe pairs of $(s_i, a_i,  r_i, s_i')$, and 
for each $i$, we augment the data with the next action $a_i'$ as part of the algorithm by drawing from $\pi(\cdot~|~s_i')$, where $\pi$ is the known target policy. The data assumption in Eq~\eqref{equ:data} is equivalent to saying that $y_i\sim \dist P_\pi(\cdot~|~ x_i)$; see Eq~\eqref{equ:ppi}. %
}\label{tab:notation}
\end{table}

\myparagraph{Notation}
We define a few notations that will simplify the presentation in the rest of work. First of all,
to facilitate the definition of the empirical Bellman operator in the sequel (see Section~\ref{sec:value} and Remark~\ref{rmk:bellman}), 
we append each $(s_i, a_i, r_i, s_i')$  with a random action $a_i'$ drawn from the target policy $\pi(\cdot ~|~s_i')$ following $s_i'$. This can be done for free as long as the target policy $\pi$ is given.  
Also, we write $x_i = (s_i,a_i)$,  $x_i' = (s_i', a_i')$, and $y_i = (x_i', r_i)=(s_i', a_i', r_i)$. 
Correspondingly, define $\set X = \set S \times \set A$ to be the state-action space. Then the observed data can be written as pairs of $\{x_i, y_i\}_{i=1}^n$. 
We equalize the data $\emp D_n$ with its empirical measure $\emp d_n = \sum_{i=1}^n \dist \delta_{x_i, y_i}/n$, where $\delta$ is the Delta measure.  See Table~\ref{tab:notation}. 
Define 
\bba \label{equ:ppi}
\dist P_{\pi}(y ~| ~x ) =\dist P(s', r ~|~x)\pi(a'~|~s')\,.
\eea 
Then Assumption~\ref{ass:data} is equivalent to saying that $y_i \sim \dist P_{\pi}(\cdot ~|~ x_i)$. %
In this way, our setting can be viewed as a supervised learning problem on $\{x_i,y_i\}$ with an unknown $\dist P_\pi$. 
The difference is that we are interested in estimating $J_{\pi,\dist P}$, 
which is a nonlinear functional of $\dist P_\pi$, 
rather than the whole model $\dist P_\pi$ (which would yield a model-based method; see discussions in Section~\ref{sec:relatedworks}).

We use the star notation to label quantities that depend on the unknown parameters (such as $J\true$), and the hat notation  to denote quantities that depend on observed data (such as $\emp D_n$ and $\hat J^\pm(\emp D_n)$).
The prime notation always denotes the next state (such as $s_i'$).

\section{Related Work}
\label{sec:relatedworks}
We give an overview of different approaches for uncertainty estimation in OPE. 
\paragraph{Finite-Horizon Importance Sampling (IS)} 
Assume the data is collected by rolling out a known behavior policy $\pi_0$ up to a trajectory length $T$,  
then we can estimate the finite horizon reward by changing $\E_{\pi, \dist P}[\cdot]$ to $\E_{\pi_0, \dist P}[\cdot]$ with importance sampling\citep[e.g.,][]{precup00eligibility,precup2001temporal,thomas15high,thomas2015highope}.
Taking the trajectory-wise importance sampling as an example,
assume we collect a set of independent trajectories $\tau_i\defeq \{s_t^i,a_t^i,r_t^i\}_{t=0}^{T-1}$, $i=1,\ldots, m$ up to a trajectory length $T$ by unrolling a known \emph{behavior policy} $\pi_0$. When $T$ is  large, we can estimate $J\true$ by a weighted averaging:
\begin{align}\label{equ:IS}
\hat J^{\rm{IS}}=\frac{1}{m}\sum_{i=1}^m \ratio(\tau_i) J(\tau_i)\,, ~~~~~
\text{where} ~~~\ratio(\tau_i) = \prod_{t=0}^{T-1} \frac{\pi(a_t^i|s_t^i)}{\pi_0(a_t^i|s_t^i)}\,,~~~~~ J(\tau_i) = \sum_{t=0}^{T-1} \gamma^t r_t^i\,.
\end{align}
One can construct non-asymptotic confidence bounds based on $\hat J^{\rm{IS}}$
using variants of concentration inequalities \citep{thomas2015safe,thomas2015highope}.  
Unfortunately, a key problem with this IS estimator is that the importance weight $\ratio(\tau_i)$ is  a product of the density ratios over time, 
and hence tends to cause an explosion in variance when the trajectory length $T$ is large. 
Although improvement can be made by using per-step and self-normalized weights \citep{precup2001temporal}, or control variates \citep{jiang16doubly, thomas16data},
the \emph{curse of horizon} remains to be a key issue to the classical IS-based  estimators \citep{liu2018breaking}. 

 Moreover, due to the time dependency between the transition pairs inside each trajectory, 
the non-asymptotic concentration bounds can only be applied on the trajectory level and hence decay with 
the number $m$ of independent trajectories in an $O(1/\sqrt{m})$ rate, though $m$ can be small in practice. 
{We could in principle apply the concentration inequalities of Markov chains \citep[e.g.,][]{paulin2015concentration} to the time-dependent transition pairs, 
but such inequalities 
require to have an upper bound of certain mixing coefficient of the Markov chain, which is unknown
and hard to construct empirically.}
Our work addresses these limitations 
by constructing a non-asymptotic bound that decay with 
the number $n = mT$ of transitions pairs, while without requiring known behavior policies and independent trajectories.

\paragraph{Infinite-Horizon,  Behavior-Agnostic OPE} 
Our work %
is closely related to 
the recent advances 
in infinite-horizon and behavior-agnostic OPE, including, for example,  
\citet{liu2018breaking, feng2019kernel, tang2019doubly, mousavi2020blackbox,liu2019understanding, yang2020off,
xie2019optimal,yin2020asymptotically}, 
as well as the DICE-family \citep[e.g.,][]{nachum2019dualdice, nachum2019algaedice,  zhang2020gendice,wen2020batch, zhang2020gradientdice}. 
These methods are based on 
either estimating the value function, or the stationary visitation distribution, 
which is shown to form a primal-dual relation  \citep{tang2019doubly,uehara2019minimax,jiang2020minimax} that we elaborate in depth in Section~\ref{sec:infinite}. 

 Besides \citet{fengaccountable2020} which directly motivated this work, 
{there has been a recent surge of interest in interval estimation under infinite-horizon OPE \citep[e.g.,][]{liu18representation,jiang2020minimax,duan2020minimax,dai2020coindice,fengaccountable2020,tang2020lipschitz,yin2020near,lazic2020maximum}.  %
For example, 
\citet{dai2020coindice} develop an asymptotic confidence bound (CoinDice) for DICE estimators with an i.i.d assumption on the off-policy data; 
\citet{duan2020minimax} provide a data dependent confidence bounds based on %
Fitted Q iteration (FQI) using linear function approximation when the off-policy data consists of a set of independent trajectories;  
\citet{jiang2020minimax}
provide a minimax 
method closely related to our method but do not provide analysis for data error; 
{\citet{tang2020lipschitz} propose a fixed point algorithm for constructing deterministic intervals of the true value function when the reward and transition models are deterministic and the true value function has a bounded Lipschitz norm.}} %

\paragraph{Model-Based Methods}
Since the model $\dist P$ is the only unknown variable, 
we can construct an estimator $\emp P$ of $\dist P$ using maximum likelihood estimation or other methods,  
and plug  it into \eqref{equ:jpi0} to obtain  a plug-in estimator  $\hat J = J_{\pi,\emp P}$. 
This yields the model-based approach to OPE~\citep[e.g.,][]{jiang16doubly,liu18representation}.
One can also estimate the uncertainty in $J_{\pi, \emp P}$ by propagating the uncertatinty in $\emp P$ \citep[e.g.,][]{asadi2018equivalence,duan2020minimax},
but it is hard to obtain  non-asymptotic 
and computationally efficient  
bounds unless $\emp P$ is assumed to be simple linear models. 
In general, estimating the whole model $\dist P$ can be 
an unnecessarily complicated problem as an intermediate step of the possibly simpler problem of estimating $J_{\pi,\dist P}$.

\paragraph{Bootstrapping, Bayes, Distributional RL}
As a general approach of uncertainty estimation, 
bootstrapping has been used in interval estimation in RL in various ways 
\citep[e.g.,][]{white2010interval, hanna2017bootstrapping, kostrikov2020statistical, hao2021bootstrapping}. 
Bootstrapping is simple and highly flexible, and can be applied to time-dependent data  (as appeared in RL)
using variants of block bootstrapping methods \citep[e.g.,][]{lahiri2013resampling, white2010interval}. %
However, bootstrapping typically only provides asymptotic guarantees;
although non-asymptotic bounds of bootstrap exist 
\citep[e.g.,][]{arlot2010some}, 
they are sophistic and difficult to use in practice
and would require to know the mixing condition for the dependent data. 
Moreover, bootstrapping is time consuming since it requires to repeat the whole off-policy evaluation pipeline on a large number of resampled data. %

Bayesian methods \citep[e.g.,][]{engel2005reinforcement,ghavamzadeh2016bayesian,yang2020offline} offer  another general approach to uncertainty estimation in RL, but require to use approximate inference algorithms and do not come with non-asymptotic frequentist guarantees.  
In addition, distributional RL \citep[e.g.,][]{bellemare2017distributional} seeks to quantify the intrinsic uncertainties inside the Markov decision process, which is orthogonal to the %
epistemic uncertainty 
that we consider in off-policy evaluation.

\section{Two Dual Approaches to Infinite-Horizon Off-Policy Estimation}
\label{sec:infinite}

The main idea of infinite-horizon OPE 
is to transform the estimation of the expected reward into estimating either the \emph{Q-function} or the \emph{visitation distribution (or its related density ratio)} %
by exploiting the stationary property of the MDP under $\pi$.  %
This section discusses these two tightly connected methods, 
which 
form a primal-dual relation and together lay out a foundation for our main confidence bounds.

The Q-function associated with policy $\pi$ and model $\dist P$ 
is defined as  
$$q\true(x) = \E_{\pi, \dist P}\left [\sum_{t=0}^\infty \gamma^t r_t ~|~ x_0=x\right ] \,,$$
where the expectation is taken when we execute $\pi$ under model $\dist P$ initialized from a fixed state-action pair $x_0=(s_0,a_0)$. 
Let $\dist d_{\pi, t}$ be the distribution of $(x_t, y_t)=(s_t,a_t,s_t',a_t',r_t)$ when executing policy $\pi$ starting from $s_0 \sim \dist d_0$ for $t$ steps.  
The visitation distribution of $\pi$ is defined as 
$$\dist d\true=\sum_{t=0}^\infty \gamma^t  \dist d_{\pi,t}\,.$$ 
Note that $\dist d\true$ integrates to $1/(1-\gamma)$, although we still treat it as a probability measure in the notation. 

The expected reward $J\true$ can be expressed using either $q\true$ or $\dist D\true$ as follows: 
 \begin{align}\label{equ:jpimac}
    J\true:= 
    \E_{\pi, \dist P}\left [\sum_{t=0}^{\infty}\gamma^{t}r_t \right ] =  \E_{r\sim \dist d\true}[r] = \E_{x\sim \dist d_{\pi,0}}[q\true(x)]\,,
\end{align}
where $r\sim \dist D\true$ (resp. $x\sim \dist D_{\pi,0}$) denotes sampling from the $r$-(resp. $x$-) marginal distribution of $\dist D\true$ (resp. $\dist D_{\pi,0}$). 
Eq.~\eqref{equ:jpimac}
 transforms the estimation of $J\true$ into estimating either $q\true$ or $\dist D\true$ and plays a key role in  infinite-horizon OPE.  %

\subsection{Value Estimation via Q Function}\label{sec:value}  
Because $\dist D_{\pi,0}(x)=\dist D_0(s)\pi(a|s)$ is known, we can estimate $J\true$ by $\E_{x\sim \dist D_{\pi,0}}[\hat q(x)]$  with any estimation $\hat q$ of the true Q-function $q\true$; the expectation under $x\sim \dist D_{\pi,0}$ can be estimated to any accuracy with Monte Carlo.  
To estimate $q\true$, we consider the empirical and expected Bellman residual operator: 
\bba \label{equ:bellman}
\hat \R q(x, \y)  =  
q(x) - \gamma q(x') - r \,, &&
    \Rtrue q(x)= 
    \E_{y\sim \dist P_{\pi}(\cdot|x)} 
    \left[ \hat \R q(x,y)\right ]\,.
\eea 
It is well-known that $q\true$ is the unique solution of the \textit{Bellman equation} $\Rtrue q =0$. %
Since $y_i \sim \dist P_{\pi}(\cdot | x_i)$ for each data point in $\data$, 
if $q = q\true$, then $\hat \R q(x_i, \y_i)$ is a \emph{zero-mean} random variable conditional on $x_i$. Let $\ratio$ be any function from $\set X$ to $\RR$, then $\sum_i \hat \R q(x_i, \y_i) \ratio (x_i)$  also has zero mean. %
This motivates the following \emph{functional} Bellman loss \citep{feng2019kernel, fengaccountable2020,xie2020q}, 
\begin{align}\label{equ:kblfunctional}
    L_{\funcset W}(q; ~ \emp D_n) & :=  
    \sup_{\ratio\in \funcset W}   \left\{
    \frac{1}{n}\sum_{i=1}^n\hat\R q({x_i, \y_i}) \ratio(x_i)  \right\},
\end{align} %
where $\funcset W$ is a set of functions $\ratio\colon \set X\to\RR$. 
To ensure that the sup is positive and finite, $\W$ is typically set to be a unit ball of some normed function space $\H$, that is, 
$$\W=\{\ratio\in\H\colon\norm{\ratio}_{\H}\leq 1\}.$$
\citet{feng2019kernel} consider the simple case when 
$\funcset W$ is  the unit ball $\K$ of the reproducing kernel Hilbert space (RKHS) with a positive definite kernel $k\colon \set X \times \set X \to \RR$, 
for which 
the loss has a simple closed form: %
\begin{align}\label{equ:kblkernel}
    L_{\K}(q; ~ \emp D_n)
    &=%
     \sqrt{ \frac{1}{n^2}\sum_{i,j=1}^n
     \hat \R q(x_i,y_i) \k(x_i, {x}_j)  \hat\R q(x_j, y_j)}\,.  
\end{align}
Note that the RHS of Eq.~\eqref{equ:kblkernel} is the \emph{square root} of the kernel Bellman V-statistics  in \citet{feng2019kernel}.  
 \citet{feng2019kernel} showed that, when the support of the data distribution $\emp D_n$ covers the state-action space $\set X$ (which requires an infinite data size when the domain size is infinite) and $k$ is an integrally strictly positive definite kernel, we have $ L_{\K}(q; ~ \emp D_n) =0$ iff $q = q\true$. Therefore, one can estimate $q\true$ by minimizing $L_{\K}(q, \emp D_n)$. 

\begin{rmk}  \label{rmk:bellman}
The empirical Bellman residual operator $\hat\R$ can be extended to 
\begin{align} \label{equ:Rqm}
\hat \R q(x,y) = q(x) - \gamma \frac{1}{m}\sum_{\ell=1}^m q(s',a_\ell')  -r,
\end{align}
where $\{a_\ell'\}_{i=1}^m$   
are i.i.d. drawn from $\pi(\cdot|s')$. %
As $m$ increases, this gives an unbiased estimator of $\Rtrue q$ with lower variance. 
If $m=+\infty$, we have 
\begin{align} \label{equ:Rqinf}
\hat \R q(x,y) = q(x) - \gamma \E_{a'\sim \pi(\cdot~|~s')} [q(s',a')] - r ,
\end{align}
which coincides with the operator used in the expected SARSA \citep{sutton98beinforcement}. Therefore, the choice of $m$ provides a trade-off between accuracy and computational cost. 
Without any modification, all results in this work can be applied to 
the $\hat \R q$ in \eqref{equ:Rqm}-\eqref{equ:Rqinf} 
for any positive integer $m$. { We use $m=5$ in most of our experiments.}%
\end{rmk}

\subsection{Value Estimation via Visitation Distribution}  
Another way to 
estimate $J\true$ in Eq.~\eqref{equ:jpimac} is to approximate $\dist D\true$ with a weighted empirical measure of the data \citep{liu2018breaking, nachum2019dualdice, mousavi2020blackbox, zhang2020gendice}. 
The key idea is to assign an importance weight $\ratio(x_i)$ to each data point $x_i$ in $\data$. We can choose the function $\ratio\colon \set X \to \RR$ properly such that $\dist D\true$ and hence $J\true$ can be approximated by the $\ratio$-weighted empirical measure of $\data$ (and reward) as follows:   
\begin{align} \label{equ:densityraio}
\dist D\true \approx \emp d_n^\ratio \defeq \frac{1}{n}\sum_{i=1}^n \ratio(x_i) \delta_{x_i,y_i},&&
J\true \approx 
\hat J_{\ratio}\defeq \E_{\emp d_n^\ratio}[r] = \frac{1}{n}\sum_{i=1}^n \ratio(x_i) r_i\,. 
\end{align}
Intuitively, $\ratio$ can be viewed as the density ratio between $\dist D\true$ and $\emp D_n$, although the empirical measure $\emp D_n$ may not have well-defined density. 
\citet{liu2018breaking,mousavi2020blackbox} proposed to estimate $\ratio$ by minimizing a discrepancy measure between $\emp d_n^\ratio$ and $\dist D\true$. %
To see this, note that $\emp d_n^\ratio = \dist D\true$ %
if $\Delta(\emp d_n^\ratio, q)=0$ for any function $q$, with  $\Delta$ defined as 
\bba \label{equ:Delta}
\Delta(\emp d_n^\ratio, q) 
& = \E_{\emp d_n^\ratio}[\gamma q(x') - q(x)] - \E_{\dist d\true}[\gamma q(x')-q(x)] \notag \\ 
& = \E_{\emp d_n^\ratio}[\gamma q(x') - q(x)] + \E_{\dist d_{\pi,0}}[q(x)],
 \eea 
where we use the fact that $ \E_{\dist d\true}[\gamma q(x')-q(x)] = -\E_{\dist d_{\pi,0}}[q(x)]$ in \eqref{equ:Delta}  \citep[see Theorem 1,][]{liu2018breaking}.  
Also note that the RHS of Eq.~\eqref{equ:Delta} can be practically calculated given any $\ratio$ and $q$ without knowing $\dist D\true$.  
Let $\mathcal Q$ be a set of functions $q\colon \set X\to \RR$. 
 One can define the following loss  for $\ratio$: 
 \begin{align} \label{equ:sw}
 &\Wd =  \sup_{q\in \funcset Q} \left\{  {\Delta}({\emp d_n^\ratio}, q) \right \}\,. 
\end{align}

Similar to $L_{\K}(q;~\emp D_n)$,  
when $\funcset Q$ equals the unit ball $\tilde \K$ of the RKHS of a positive definite kernel $\tilde k(x,\bar x)$,  
Eq.~\eqref{equ:sw} can be expressed into a quadratic closed form shown in \citet{mousavi2020blackbox}: 
 \bba \label{equ:Wfkernel} 
  \Wd   %
  = \sqrt{A + 2B + C}, 
 \eea  
 with   
 \begin{align*}
 & A = \E_{(x,\bar x)\sim \D_{\pi,0}\times \D_{\pi,0}}\left  [\k(x, \bar x) \right ]\,, \\
       & B = \E_{(x,\bar x)\sim  \emp D_{n}^\ratio \times \D_{\pi,0} } \left [\hat{\opt T}^x  \k(x, \bar x) \right ] \,,\\
       & C = \E_{(x,\bar x)\sim  \emp D_{n}^\ratio  \times  \emp D_{n}^\ratio  } 
       \left [ \hat{\opt T}^x  \hat{\opt T}^{\bar x}  \k(x, \bar x) \right] \,,
 \end{align*}
 where $ \hat{\opt T}^x f(x) = \gamma f(x') - f(x)$, and 
 $\hat{\opt T}^x  \hat{\opt T}^{\bar x}  \k(x, \bar x)$ is obtained by applying
 $ \hat{\opt T}^{\bar x}$ and $\hat{\opt T}^{ x}$  sequentially by treating $k$ as a function of $\bar x$ and then of $x$.

\subsection{Primal and Dual Deterministic Bounds (Infinite Data Case)}
\label{sec:pop_bounds}

The  two types of estimators above are connected in a primal-dual fashion, which can be used to build deterministic upper and lower bounds of $J_*$ that can be viewed as the limit of our empirical bounds with an infinite data size $n$ \citep{tang2019doubly, jiang2020minimax}.   
We provide an overview of this framework 
and derive a number of bounds whose empirical version will be discussed in depth in Section~\ref{sec:main}. 

Let $\set Q$ be a function set that includes the true Q-function $q\true$. 
It is known that the $J\true$ can be represented using the following optimization problem: 
\bba \label{equ:infinite_oracleup}
J\true = J_{\funcset Q, *}^+ := \sup_{q \in \Q} \left\{ \E_{\dist D_{\pi,0}} [q] ~~~~s.t.~~~~ \Rtrue q (x) = 0, \quad \forall x\in \set X \right \}, 
\eea
which holds because $J_* = \E_{\dist D_{\pi,0}}[q_*]$  and 
$q=q_*$ is the unique solution of $\R q(x) =0,~\forall x\in \set X$.  

Using Lagrange duality, Eq~\eqref{equ:infinite_oracleup} is equivalent to 
\begin{align} \label{equ:jtrueone}
J\true
&  = 
\sup_{q \in \Q}   \inf_{\ratiozero} 
\left \{  \E_{\dist D_{\pi,0}} [q] -  \E_{x\sim \dist D_\infty}\left [ \ratiozero(x) \R q(x) \right]
\right\}\,, 
\end{align}
where 
$\ratiozero(x) \dist D_\infty (x)$ serves as the Lagrange multiplier: 
$\dist D_\infty$ is a fixed distribution whose support is $\set X$
and $\ratiozero \colon \X \to \RR$ is optimized in the set of all functions such that the objective is defined. 

Now we extend $\dist D_\infty$ to be a joint distribution on $\set X \times \set Y$
by defining $\dist D_\infty(x,y) = \dist D_\infty(x)\dist P_\pi(y~|~x)$. 
Note that $\R q(x) = \E_{y\sim \dist P_\pi(\cdot|x)}[\hat \R q(x,y)]$ and hence 
$\E_{\dist D_\infty} [\ratiozero(x)\hat\R q (x,y)] = \E_{\dist D_\infty} [\ratiozero(x)\Rtrue q (x)]$. Therefore, %
\bba \label{equ:jtrueminimaxM} 
J\true
&  = 
\sup_{q \in \Q}   \inf_{\ratiozero} 
\left \{ M(q, \ratiozero;~\dist D_\infty) \defeq \E_{\dist D_{\pi,0}} [q] - \E_{\dist D_\infty} [\ratiozero(x)\hat\R q (x,y)] \right\}. 
\eea  
Although $\dist D_\infty$  
was technically introduced  as a part of the Lagrange multiplier, 
it should be intuitively viewed as a population data distribution, the limit of the empirical data  $\emp D_n$ as $n\to \infty$; 
in practice, we  replace $\dist D_\infty$ with $\emp D_n$ when using the bounds in this section. 
However, note that Assumption~\ref{ass:data}  does not prescribe the existence of such  $\dist D_{\infty}$ from the data $\emp D_n$ (since it does not consider the limit  when $n\to\infty$).
{We need additional assumptions, such as 
 when $(x_i,y_i)_{i=1}^n$ is i.i.d. or follows an ergodic Markov chain 
 to relate $\emp D_n$ with a well defined limit $\dist D_n$. %
 Such additional assumption is needed when ensuring \emph{a-priori} bounds on the length of the confidence interval that we construct; see Theorem~\ref{thm:uppergap}. 
 }

We can obtain an upper bound of $J\true$ from the minimax representation in Eq.~\eqref{equ:jtrueminimaxM} 
by constraining the optimization domain of $\ratiozero$ to  a normed function space $\Wo$ whose unite ball is $\W$, %
\bba \label{equ:uperbound11}
J\true %
\leq  
J_{\funcset Q, \funcset W, *}^+ \defeq
\sup_{q \in \Q}   \inf_{\ratiozero \in \Wo} M(q, \ratiozero; ~ \dist D_\infty) 
\leq \inf_{\ratiozero\in \Wo}  \sup_{q \in \Q}   M(q, \ratiozero; ~ \dist D_\infty),
\eea 
where the first inequality is due to the constraint on $\Wo$; 
the second inequality is due to exchanging the order of $\sup$ and $\inf$, which turns to equality if strong duality holds. %

Importantly, the Lagrange function $M(q, \ratio;~ \dist D_\infty)$ is connected to both $ L_{\W}(q, ~ \dist D_\infty)$ and 
$I_{\Q}(\ratio;~\dist D_\infty)$, which
allows us to derive a pair of primal and dual bounds whose non-asymptotic version will be derived in this work. 
\begin{lem} \label{lem:Mdualprimal}
Assume { $\set \H = \{ \lambda \ratio \colon \ratio \in \set W, ~ \lambda \geq 0 \}$ } and define $\dist D_{\infty}^\ratio(x,y) = \ratio(x) \dist D_\infty(x,y)$. We have %
\begin{align} %
\inf_{\ratiozero \in \set \H}M(q,\ratio;~ \dist D_\infty ) & = \inf_{\lambda \geq 0}\E_{\D_{\pi,0}}[q] - \lambda L_{\W}(q, ~ \dist D_\infty),  \label{equ:MLw}\\
\sup_{q \in \set Q}M(q,\ratio; ~ \dist D_\infty) & = \E_{\dist D_\infty^\ratio} [r] + I_{\Q}(\ratio;~\dist D_\infty). \label{equ:MLq}
\end{align}
\end{lem}

\begin{table}[t]
    \centering
    \begin{adjustbox}{width=\columnwidth,center}
    \renewcommand{\arraystretch}{2.0}
    \begin{tabular}{l|lc|l}
        \toprule
         & \bf{Expressions} && \bf{Remark} \\
        \midrule\midrule
        $J\true$ & %
        $\E_{x\sim \dist d_{\pi,0}}[q\true(x)]$ & \eqref{equ:jpimac}& Ground truth\\
        \midrule
        $J_{\funcset Q, *}^+$ &  %
$ \sup_{q \in \Q} \left\{ \E_{\dist D_{\pi,0}} [q] ~~~~s.t.~~~~ \Rtrue q (x) = 0, \quad \forall x\in \set X \right \}$ &
\eqref{equ:infinite_oracleup}&
$J\true = J_{\Q, *}^{+}$ if $q_{\true} \in \Q$ \\
        \midrule
        $ J_{\funcset Q, \funcset W, *}^+$ & %
        $\sup_{q \in \Q}  
\left \{ 
\E_{\dist D_{\pi,0}} [q] ~~s.t.~  L_{\set W}(q, ~ \dist D_\infty) \leq L_{\W}(q\true; \dist D_{\infty}) = 0 \right\} $ & \eqref{equ:qinfbound}& Primal upper bound, $ J_{\funcset Q, *}^+\leq  J_{\funcset Q, \funcset W, *}^+$\\
        \midrule 
        $ J_{\funcset Q, \funcset W, *}^{++}$ & %
        $\inf_{\ratio\in \Wo}\left\{\E_{\dist D_\infty^\ratio} [r] + I_{\Q}(\ratio;~\dist D_\infty)\right\}$ & \eqref{equ:winfbound} & Dual upper bound, $J_{\Q,\W, *}^+\leq J_{\funcset Q, \funcset W, *}^{++}
        $\\
        \midrule 
        & %
        $\hat J_{\text{dr}}  \defeq M(\hat q, \hat\ratio; ~\emp D_n)$
        & & Doubly robust estimator in \citet{tang2019doubly} \\
        \bottomrule
    \end{tabular}
    \end{adjustbox}
    \caption{{Summary of the different upper bounds in Section \ref{sec:pop_bounds}  based on $\dist D_{\infty}$. }}
    \label{tab:pop_summary}
\end{table}

Therefore, plugging Eq.~\eqref{equ:MLw} into Eq.~\eqref{equ:uperbound11}, we get 
\bb 
J_{\set Q, \set W, *}^+
& = \sup_{q \in \Q}   \inf_{\ratiozero \in \set \H} 
\left \{ 
M(q, \ratio;~\dist D_\infty)
\right\} \\
& =  \sup_{q \in \Q}  \inf_{\lambda \geq 0} 
\left \{ 
\E_{\dist D_{\pi,0}} [q] - \lambda L_{\set W}(q, ~ \dist D_\infty) \right\}. 
\ee 
By recognizing $\lambda$ as a scalar Lagrange multiplier, we have 
\bba 
J_{\set Q, \set W, *}^+ =
\sup_{q \in \Q}  
\left \{ 
\E_{\dist D_{\pi,0}} [q] ~~~~~s.t.~~~~~  L_{\set W}(q, ~ \dist D_\infty) \leq 0 \right\} , \label{equ:qinfbound}
\eea
which can be seen as a relaxation of Eq.~\eqref{equ:infinite_oracleup} since 
$L_{\set W}(q, ~ \dist D_\infty) \leq 0$ (equivalent to $L_{\set W}(q, ~ \dist D_\infty) = 0$) can be a weaker constraint than $\{\R q(x)=0,~~\forall x\in \X\}$ if $\set W$ is  a small function set. 

On the other hand, plugging Eq.~\eqref{equ:MLq} into Eq.~\eqref{equ:uperbound11}, we have 
\bba
J_{\funcset Q, \funcset W, *}^+   
& \leq   \inf_{\ratio \in \set \H} \sup_{q\in \Q} M(q, \ratio; ~ \dist D_\infty) \notag \\ 
 & = \inf_{\ratio \in \set \H} \E_{\dist D_\infty^\ratio} [r] + I_{\Q}(\ratio;~\dist D_\infty) %
  {:= J_{\funcset Q, \funcset W, *}^{++}\,.   }   \label{equ:winfbound}
\eea 
Here  Eq.~\eqref{equ:qinfbound} and 
Eq.~\eqref{equ:winfbound} are dual to each other and 
provide bounds of $J\true$ in terms of $q$ and $\ratio$, respectively. 
In Section~\ref{sec:main}, we provide empirical variants of 
Eq.~\eqref{equ:qinfbound} and 
Eq.~\eqref{equ:winfbound} %
which replace 
$\dist D_\infty$ with $\emp D_n$ while providing non-asymptotic bounds.

\myparagraph{Doubly Robust Estimation
and the Lagrangian}
Another related key feature of the Lagrange function $M(q, \ratio;~\dist D_\infty)$ above is the following ``double robustness'' property: 
\begin{align} \label{equ:doubly}
J\true = M(q\true, \ratio; ~ \dist D_\infty) = M(q, \ratio\true;~ \dist D_\infty),~~~\forall \ratio, q\,,
\end{align}
where $q\true$ is the true Q-function, and $\ratio\true$ is the density ratio between $\dist D\true$ and $\dist D_\infty$ such that $\dist D_\infty^\ratio = \dist D\true$. 
The double robustness in Eq.~\eqref{equ:doubly} says that 
$M(q, \ratio; ~ \dist D_\infty)$ equals   $J_*$  if either $q = q\true$ or $\ratio = \ratio\true$ holds.  
Therefore, 
let  $\hat q$ and  $\hat \ratio$ be estimations of $q\true$ and $\ratio\true$ respectively,  
then 
$\hat J_{\text{dr}}  \defeq M(\hat q, \hat\ratio; ~\emp D_n)$  
yields doubly robust estimation of $J\true$  
in the sense that $\hat J_{\text{dr}} $ forms an accurate estimation of $J\true$ if either $\hat q$ or $\hat \ratio$ is accurate; see more discussion in \citet{tang2019doubly}.

\qiangremoved{ 
Notice that, instead of the off-policy \emph{value} estimation, here we focus on solving the two key problems we propose for \textit{accountable} off-policy evaluation,
e.g. given $n$ off-policy data $\D=\{s_i, a_i, r_i, s_i'\}_{1 \leq i \leq n}$,
we want to construct a tight interval $[\hat{J}^{+}, \hat{J}^{-}]$ that contains the expected total discounted reward $J\true$ with high probability.
}

\qiangremoved{ 
\textbf{The Issue of Double Sampling}~~
When the transition of the MDP is stochastic, kernel loss does not suffer from the double sampling problem (i.e. at least two independent sample pair $(r, s')$ for the same state-action pair $x$ to provide consistent and unbiased estimation,
while prior algorithms such as residual gradient \citep{baird95residual} give biased estimation without double samples. \red{this should go earlier. And explained more explicitly (with equations, not words)}\red{check how it is done in feng etal.}
}

All the bounds in this Section depend on $\dist D_\infty$ which need to be replaced by the empirical data $\emp D_n$ in practice. 
One difficulty, however, is that Assumption~\ref{ass:data} does not directly imply that $\emp D_n$ converges to a fixed  limit distribution $\dist D_n$ as $n\to \infty$. 
Therefore, special care is taken in Section~\ref{sec:concentration} and Section~\ref{sec:main} to sidestep the introduction of $\dist D_\infty$ as we construct non-asymptotic confidence bounds dependent on $\emp D_n$.

\section{Concentration Inequality of Kernel Bellman Loss}    
\label{sec:concentration}
To quantify the data error, 
we establish in this section a concentration inequality to bound the deviation of the kernel Bellman loss (KBL) $L_{\K}(q\true; ~ \emp D_n)$ away from zero, which serves as a key building block 
of the non-asymptotic bounds in Section~\ref{sec:main}. 
This inequality holds under the mild data assumption \ref{ass:data} and does not require additional i.i.d.  or independence assumption thanks to a martingale structure implied Assumption~\ref{ass:data}.

We first introduce the following \emph{semi-expected} kernel Bellman loss (KBL) which 
$L_{\K}(q; ~ \emp D_n)$ concentrates around for any $q$: 
\begin{align}\label{equ:kblcond}
L^*_\K(q;~ \emp D_n) =\sqrt{\frac{1}{n^2}\sum_{ij=1}^n \opt{R} q(x_i) k(x_i, x_j) \opt{R} q(x_j)}\,,
\end{align}
where we replace the empirical Bellman residual operator $\hat \R q$ in Eq.~\eqref{equ:kblkernel} with its expected counterpart $\Rtrue q$, 
but still keep the empirical average over $\{x_i\}_{i=1}^n$ in $\emp D_n$. 
For a more general function set $\W$, we can similarly define $L_{\W}^*(q; ~ \emp D_n)$ by replacing $\hat \R q$ with $\Rtrue q$ in Eq.~\eqref{equ:kblfunctional}. 
Note that we have $L^*_\W(q;~ \emp D_n)=0$ when $q=q\true$ for any $\emp D_n$ and any $\set W.$ 

Theorem~\ref{thm:Lqpimain2} below shows that $L_{\K}(q;~ \emp D_n) $ concentrates around $L_{\K}^*(q;~ \emp D_n)$ with an $\bigO({n}^{-1/2})$ error under Assumption~\ref{ass:data}.  %
At a first glance, it may seem surprising that the concentration bound can hold even without any independence assumption between $\{x_i\}$. 
An easy way to make sense of this is by recognizing that the randomness in $y_i$ conditional on $x_i$ is aggregated through averaging, even if $\{x_i\}$ are deterministic.

\begin{thm}\label{thm:Lqpimain2}
Assume $\K$ is the unit ball of RKHS with a positive definite kernel $\k(\cdot,\cdot)$.  
Let $c_{q,k} \defeq \sup_{x,y} (\opt {\hat R} q(x,y)-\opt R q(x))^2 k(x,x) <\infty$.  
Under Assumption~\ref{ass:data}, 
for any $\delta \in (0,1)$,  %
 with at least  probability $1-\delta$, we have 
\bba \label{equ:fastcon2}
\abs{{ L_\K(q;~ \emp D_n)} -  { L^*_\K(q;~ \emp D_n)}}
\leq \sqrt{\frac{2c_{q,k} \log(2/\delta)}{n}}\,. 
\eea 
In particular, when $q = q\true$, we have $c_{q\true,k} = \sup_{x,y} (\opt {\hat R} q\true(x,y))^2 k(x,x)$, and 
\bba \label{equ:fastcon1}
L_{\K}(q\true;~ \emp D_n) \leq \sqrt{\frac{2c_{q\true,k}  \log(2/\delta)}{n}}\,. 
\eea 
\end{thm}
An upper bound of the coefficient $c_{q,k}$ can be calculated easily in practice; see \citet{fengaccountable2020} and Appendix~\ref{sec:calcq}.

Intuitively, to see why we can expect an $\bigO(n^{-1/2})$ bound, note that $L_{\K}(q,\emp D_n)$ consists of the square root of the product of two $\hat \R q$ terms, each of which contributes an $\bigO(n^{-1/2})$ error w.r.t. $\Rtrue q$.  
 Technically, 
 the proof %
 is made possible by observing that 
 Assumption~\ref{ass:data} ensures that $\{Z_i \defeq \opt{\hat R} q(x_i, y_i) - \opt R q(x_i)\colon~ i=1,\ldots, n\}$ forms a \emph{martingale difference} sequence w.r.t. $\{\emp D_{<i}\cup \{x_i\}\colon ~\forall i = 1,\ldots,n\}$, in the sense that $\E[Z_i ~|~\emp D_{<i}\cup \{x_i\}] = 0$, $\forall i$. 
The proof also leverages a special property of RKHS and applies a Hoeffding-like inequality by \citet{pinelis1992approach} on Hilbert spaces. See Appendix~\ref{sec:proofcon} for details. 
For other more general function sets $\W$, we 
establish in Appendix~\ref{sec:proofrade}
a similar bound by using Rademacher complexity, although  
it  requires to 
know an upper bound  of the Rademacher complexity of a function set associated with $\set W$ 
and yields a less tight bound than Eq.~\eqref{equ:fastcon2} if $\W = \K$. 
 
 If $\emp D_n$ weakly converges to a limit $\dist D_\infty$ as $n\to\infty$, we can expect that $L_\K(q, \emp D_n)$  converges to $L_\K(q, \dist D_\infty)$. However, $\dist D_\infty$ is not implied from Assumption~\ref{ass:data}.  
Theorem~\ref{thm:Lqpimain2} sidesteps the introduction of $\dist D_\infty$ thanks to the semi-expected KBL $L_\K^*(q, \emp D_n)$. 
Also, as Assumption~\ref{ass:data} does not impose any (weak) independence between $\{x_i\}$, without introducing further assumptions, we cannot establish that $L_\K(q;~ \emp D_n)$ concentrates around the \emph{full expectation}  $\E_{\emp D_n}[L_\K(q;~ \emp D_n)^2]^{1/2}$ where $\emp D_n$ is averaged w.r.t. the underlying data generation distribution $\dist D_{1:n}^{\diamond}$.

\section{Primal-Dual Non-Asymptotic Confidence Bounds}\label{sec:main} 

We are ready to extend the bounds in Eq.~\eqref{equ:qinfbound}-Eq.~\eqref{equ:winfbound} to finite data case.  
To avoid introducing $\dist D_\infty$, 
we start with building an empirical counterpart of bound \eqref{equ:infinite_oracleup} in Section~\ref{sec:oracle}, and then proceed to derive the non-asymptotic counterparts of bound  
\eqref{equ:qinfbound} in Section \ref{sec:primal} and  
bound \eqref{equ:winfbound} in Section \ref{sec:dual}. %

\subsection{A Data-Dependent Oracle Bound} \label{sec:oracle}
Let $\funcset Q$ be a function set that contain the true Q-function $q\true$, that is, $q\true \in \funcset Q$. 

Given a dataset $\emp D_n$, we have the following upper bound of $J\true$ that generalizes \eqref{equ:infinite_oracleup}:
\bba \label{equ:oracleup}
\hat J_{\funcset Q, *}^+ = \sup_{\blue{q}\in \funcset Q} \left \{  
\E_{\D_{\pi,0}}[\blue{q}] ~~~~s.t.~~~~ \hat {\opt R} \blue{q}(x_i,y_i) = \hat {\opt R} q\true(x_i, y_i), ~~~~ \forall i =1,\ldots, n \right\}\,.  
\eea
By definition, 
this 
is the tightest upper bound given only the information in $\emp D_n$ through 
the empirical Bellman operator, 
as in this case  $q$ and $q\true$ would look indistinguishable if $\hat {\opt R} \blue{q}(x_i,y_i) = \hat {\opt R} q\true(x_i,y_i)$ for all $i$.
The lower bound $\hat J_{\funcset Q, *}^-$ can be defined analogously by replacing $\sup_{\blue{q}\in \funcset Q} $ with $\inf_{\blue{q}\in \funcset Q}.$ 
Note that $\hat J_{\funcset Q, *}^+$ and $\hat J_{\funcset Q, *}^-$ are still not practically computable because they depends on the unknown $q\true$.

\begin{pro}\label{pro:obviously}
Assume $q\true \in \Q$, we have $J\true \in \left [\hat J_{\funcset Q, *}^-, ~\hat J_{\funcset Q, *}^+\right ]$.  
\end{pro}
Note that this result holds trivially because $q\true$ is included in the optimization domain in Eq.~\eqref{equ:oracleup}. It \emph{does not} require any assumption on the data $\emp D_n=\{x_i,y_i\}$. That is, 
$J\true \in  [\hat J_{\funcset Q, *}^-, ~\hat J_{\funcset Q, *}^+]$ remains to be true (but may not be useful) even if $\emp D_n$ is filled with  random numbers irrelevant to $\dist P$ and $\pi$; 
the bound becomes useful if 
we have sufficiently number of data points satisfying $y_i \sim \dist P_{\pi}(\cdot~|~x_i)$.

\paragraph{Introducing $\Q$ is Necessary}
 It is necessary to introduce the function $\Q$ to ensure a finite bound unless $\set S$ is a small discrete set. 
Removing the  constraint  of $q\in \Q$ 
in Eq.~\eqref{equ:oracleup} would lead to an infinite upper/lower bound, unless the $\{s_i,s_i'\}_{i=1}^n$ pairs from the data $\emp D_n$ almost surely covers the whole state space $\set S$. %

\begin{pro} \label{pro:free}
Define $\Q_{\pi,\infty} = \{\blue{ q}\colon ~ \hat{\opt R} \blue{q}(x_i, y_i) = \hat{\opt R} q\true(x_i,y_i), ~~\forall i = 1,\ldots, n\}.$
Then, 
unless $\prob_{s\sim \dist d_{\pi,0}}(s\notin \{s_i, s_i'\}_{i=1}^n) = 0$,  we have 
\bb 
\inf_{\blue{q}\in \Q_{\pi,\infty}} \E_{\dist D_{\pi,0}}[\blue{q}] = -\infty, &&
\sup_{\blue{q}\in \Q_{\pi,\infty}} \E_{\dist D_{\pi,0}}[\blue{q}] = +\infty.
\ee 
\end{pro}
Note that $\prob_{s\sim \dist d_{\pi,0}}(s\notin \{s_i, s_i'\}_{i=1}^n) = 0$ can hold only when the data size $n$ is no smaller than the cardinality of the state space $\set S$, which is infinite when $\set S$ is a continuous domain. 
Therefore,
it is necessary to 
introduce a $\Q$ that is smaller than $\Q_{\pi,\infty}$  unless $\set S$ is a  discrete set with a small number of elements.

Note that every element of the $\Q_{\pi,\infty}$ defined in Proposition~\ref{pro:free} is indistinguishable with $q\true$ 
from the information accessible through the empirical Bellman operator $\hat \R$. 
Thus, unless $\Q = \Q_{\pi,\infty}$, which would make it too large to be useful, 
we can not provably guarantee that $q\true \in \Q$, because every element in $\Q_{\pi,\infty}\setminus \Q$ can have a chance to be the true  Q-function $q\true$, where $\setminus$ denotes the set minus operator.   
Therefore, the correctness of our bound will ultimately relies on an un-checkable model assumption, %
which is unavoidable in {many statistical estimation problems} in general. 

On the other hand, %
it is possible  to empirically reject a poorly chosen $\Q$ by hypothesis testing  when $\Q \cap \Q_{\pi,\infty} = \emptyset$. 
Consider the following test: 
$$
\text{Null:}~~ q\true \in \set Q  ~~~~~~vs.~~~~\text{Alternative}: ~~q\true \not\in \set Q. 
$$
We can reject $~ q\true \in \set Q$ with a false positive error $\delta$ if 
$\inf_{q\in \set Q} L_{\W}(q;~\emp D_n) \geq \varepsilon_n$, where $\varepsilon_n$ satisfies $\prob(L_{\funcset W}(q\true;~ \emp D_n)\leq \varepsilon_n) \geq 1-\delta$. 
This is because %
\bb 
\prob\left (\text{Reject $q\true \in \set Q$}~\bigg |~ \text{$q\true \in \set Q$ is true}\right )
& = \prob\left ( \inf_{q\in \set Q} L_{\W}(q;~\emp D_n) \geq \varepsilon_n~\bigg |~ \text{$q\true \in \set Q$ is true}\right )  \\
& \leq  \prob\left ( L_{\W}(q\true;~\emp D_n) \geq \varepsilon_n\right )\\
& \leq \delta. 
\ee
{When $\set Q$ is a finite ball in RKHS, 
we can practically solve $\min_{q\in \set Q} L_{\W}(q;~\emp D_n)$ by using the finite representer theorem of RKHS \citep{scholkopf2018learning}; see related discussion in Section~\ref{sec:primal} and Appendix~\ref{sec:kernelw}. 
}

\subsection{Empirical Primal Bound via Functional Bellman Loss}  \label{sec:primal} 
To make use of the concentration inequality in Section~\ref{sec:concentration}, we relax Eq.~\eqref{equ:oracleup} to the following bound that depends on a function set $\W$: 
\begin{align} \label{equ:oracleW}
\hat J_{\funcset Q, \funcset W, *}^+ = \sup_{\blue{q}\in \funcset Q} \left \{  
\E_{\D_{\pi,0}}[\blue{q}] ~~~~s.t.~~~~ L_\W(\blue{q};~\emp D_n) \leq   L_\W({q\true};~\emp D_n) 
\right\}\,,
\end{align}
and $\hat J_{\funcset Q, \funcset W, *}^-$ is defined analogously. 
Obviously, %
$ \hat J_{\funcset Q, \W, *}^+$ is an upper bound of $ \hat J_{\funcset Q, *}^+$ 
because the optimization domain in Eq.~\eqref{equ:oracleW} is  larger than that of Eq.~\eqref{equ:oracleup}.  
\begin{pro}\label{pro:zero} 
For any $\Q$, $\W$ and data $\emp D_n$, we have 
\begin{align}\label{equ:trivialQ}
    \left [\hat J_{\funcset Q, *}^-,~~ \hat J_{\funcset Q, *}^+\right] \subseteq \left [\hat J_{\funcset Q, \funcset W, *}^-,~~ \hat J_{\funcset Q, \funcset W, *}^+\right].  
\end{align}
\end{pro}

Note that $\hat J_{\Q, \W,*}^+$ is still not practically computable because it depends on 
the unknown $L_\W({q\true};~\emp D_n)$. 
However, because 
$L_\W({q\true};~\emp D_n) $ concentrates around zero when $\emp D_n$ satisfies 
Assumption~\ref{ass:data}, 
we can construct a fully empirical  bound  
as follows: %
\begin{align} \label{equ:primalbound}
\hatJup_{\funcset Q, \funcset W} = 
\sup_{\blue{q} \in \funcset Q} 
\left \{ 
\E_{\dist D_{\pi,0}}[\blue{q}] 
~~~s.t.~~~  L_\W(\blue{q};~\emp D_n) \leq {\varepsilon_n} 
\right \}, 
\end{align}
where $\varepsilon_n$ is a positive number properly chosen such that 
\bba \label{equ:tail}
\prob(L_{\funcset W}(q\true;~ \emp D_n)\leq \varepsilon_n) \geq 1-\delta\,.
\eea 
We  set $\varepsilon_n = \sqrt{2c_{q\true,k}\log(2/\delta)/n}$ when $\W = \K$ following Theorem~\ref{thm:Lqpimain2}. 
The lower bound $ \hatJlow_{\Q,\W}$ can be defined analogously. 
Obviously, Eq.~\eqref{equ:primalbound} is the empirical counterpart of Eq.~\eqref{equ:qinfbound}.

\begin{pro}\label{pro:Jstarbound}
Assume %
Eq.~\eqref{equ:tail} holds. 
Then for any  set $\set Q$, %
we have 
\begin{align} \label{equ:Jstarbound}
\prob \left (\left [\hat J^-_{\set Q, \funcset W, *}, ~~ \hat J^+_{\set Q, \funcset W, *} \right ] 
~ \subseteq ~  
\left [\hat J^-_{\set Q, \set W} , ~~ \hat J^+_{\set Q, \set W}  \right ] 
\right) \geq 1-\delta\,. 
\end{align}
 Further, if  {$q\true \in \Q$}, we have 
 \begin{align}
\prob\left (%
J\true  \in 
\left [\hat J^-_{\set Q, \set W} , \hat J^+_{\set Q, \set W}  \right ] \right) \geq 1- \delta\,.
\end{align}
\end{pro}

\begin{table}[t]
    \centering
    \begin{adjustbox}{width=\columnwidth,center}
    \renewcommand{\arraystretch}{1.8}
    \begin{tabular}{l|lc|l}
        \toprule
         & \bf{Expressions} && \bf{Remark} \\
        \midrule\midrule
        $J\true$ & %
        $\E_{x\sim \dist d_{\pi,0}}[q\true(x)]$ &\eqref{equ:jpimac}& Ground truth\\
        \midrule
        $\hat J_{\funcset Q, *}^+$ &  %
        $\sup_{q\in \funcset Q} \left \{  
\E_{\D_{\pi,0}}[q] ~~s.t.~~ \hat {\opt R} q(x_i,y_i) = \hat {\opt R} q\true(x_i, y_i), ~ \forall i\in [n] \right\}$& \eqref{equ:oracleup} & Oracle upper bound, if $q\true\in \Q$, $J\true \leq \hat J_{\funcset Q, *}^+$\\
        \midrule
        $\hat J_{\funcset Q, \funcset W, *}^+$ & $ %
        \sup_{q\in \funcset Q} \left \{  
\E_{\D_{\pi,0}}[q] ~~s.t.~~ L_\W(q;~\emp D_n) \leq   L_\W({q\true};~\emp D_n) 
\right\}$ &\eqref{equ:oracleW}& Oracle upper bound, $\hat J_{\funcset Q, *}^+\leq \hat J_{\funcset Q, \funcset W, *}^+$\\
        \midrule 
        $\hatJup_{\Q,\W}$ & %
        $\sup_{q\in \funcset Q} \left \{
\E_{\D_{\pi,0}}[q] ~~s.t.~~ L_\W(q;~\emp D_n) \leq \varepsilon_n \right\}$ & \eqref{equ:primalbound} & Primal bound, $\prob \left (\hat J^+_{\set Q, \funcset W, *} \leq \hat J^+_{\set Q, \set W}  \right) \geq 1-\delta$\\
        \midrule 
        $\hat J_{\Q,\W}^{++}$ & 
        $\inf_{\ratio\in \Wo}\left\{ \E_{\emp D_n^\ratio}[r] + \Wd  +  \varepsilon_n \norm{\ratio}_{\H} \right\}$ & \eqref{equ:mainVbounddelta} & Dual bound, 
                $\hatJup_{\Q,\W}\leq \hat J_{\Q,\W}^{++}$ \\ %
        \bottomrule
    \end{tabular}
    \end{adjustbox}
    \caption{Summary of different upper bounds in Section \ref{sec:main} based on $\emp D_n$, which are empirical counterparts of the bounds in Table~\ref{tab:pop_summary}. $\hatJup_{\Q,\W}$ and $\hat J_{\Q,\W}^{++}$ can be calculated from empirical data, while $\hat J_{\funcset Q, *}^+$ and  $\hat J_{\funcset Q, \set W, *}^+$ cannot since they depend on the unknown $q\true$.} 
    \label{tab:estimator_summary}
\end{table}

\paragraph{Computation of $\hat J^+_{\Q, \W}$ in Eq.~\eqref{equ:primalbound}}
  If $\Q$ is taken to be an RKHS ball, 
the optimization in Eq.~\eqref{equ:primalbound} can be shown to reduce to a finite dimensional convex optimization, and hence can be solved in practice. 
Precisely, if $\Q$ is a finite ball of the RKHS associated with a positive definite kernel $\tilde k(\cdot,\cdot)$ (which should be distinguished with kernel $k(\cdot,\cdot)$ of $\K$), 
then 
by the finite representer theorem of RKHS \citep{scholkopf2018learning}, 
 the global optimum of \eqref{equ:primalbound} can be achieved by a function of form 
$$
q(x) = \sum_{i=1}^n \alpha_i \tilde k(x, x_i)\,. 
$$%
Plugging this into Eq.~\eqref{equ:primalbound}, 
the optimization can be shown to reduce to a convex optimization on $\{\alpha_i\}_{i=1}^n$ with 
 a linear objective and quadratic inequality constraint.

Unfortunately, 
when the data size $n$ is large, solving Eq.~\eqref{equ:primalbound} 
still leads to a high computational cost. 
Importantly, because the guarantee in Eq.~\eqref{equ:Jstarbound} only holds
when the maximization is solved to global optimality, 
fast approximation methods, such as random feature approximation \citep{rahimi2007random}, 
should not be used in principle. 
We address this problem by considering the dual form of Eq.~\eqref{equ:primalbound}, which avoids to solve the challenging global optimization in Eq.~\eqref{equ:primalbound}.  
Moreover, the dual form enables us to better understand the tightness of the confidence interval and issues regarding the choices of $\Q$ and $\W$.

 \subsection{The Dual Bound}  
 \label{sec:dual}

To derive the dual bound, let us plug the definition of $L_\W(q;~\emp D_n)$ into Eq.~\eqref{equ:kblfunctional} and introduce a Lagrange multiplier $\lambda$: 
\begin{align} 
\hatJup_{\Q,\W}  
& =  \sup_{q\in \funcset Q} \inf_{h\in \W} \inf_{\lambda \geq 0}
\E_{\D_{\pi,0}}[q] - \lambda\left (\frac{1}{n}\sum_{i=1}^n h(x_i) \opt{\hat R}q(x_i,y_i) - \varepsilon_n\right) \notag\\
& =  \sup_{q \in \Q} \inf_{\ratio \in \H}
\left\{ %
\E_{\D_{\pi,0}}[q] - \frac{1}{n}\sum_{i=1}^n \ratio(x_i) \opt{\hat R}q(x_i) +  \varepsilon_n \norm{\ratio}_{\H} \right\}\,,%
\label{equ:maxminM}
\end{align}
 where we assume that $\set W$ is the unit ball of the normed space $\Wo$
 and hence can write any $\ratio$ in $\Wo$ into $\ratio(x) = \lambda h(x)$ such that  $h\in \set W$ and 
 $\lambda =   ||w ||_{\Wo}$. %
 Exchanging the order of min/max and some further derivation yields the following main result.

\begin{algorithm}[t]
\caption{Non-asymptotic Confidence Interval for Off-Policy Evaluation}
\label{algo:dual_bound}
\begin{algorithmic}
\STATE {\bfseries Input:} 
Off-policy data $\emp D_{n} = (s_i, a_i, r_i, s_i^{\prime})_{i=1}^{n} $; 
discounted factor $\gamma$; 
an RKHS $\Wo$ with kernel $k(\cdot, \cdot )$;  
a finite ball $\set Q$ in the RKHS with kernel $\tilde{k}(\cdot, \cdot)$; 
significance level $\delta \in (0,1)$.
\STATE

\STATE In Eq~\eqref{equ:mainVbounddelta}, set $\varepsilon_{n} = \sqrt{2c_{q\true, k}\log(2 / \delta) / n}$ with $c_{q\true, k}$ calculated in Eq.~\eqref{equ:c_qk}.
\STATE Approximately solve $\omega_{+} = \arg\min_{\omega \in \Wo} \hat F^+_\Q(\omega)$, and 
$\omega_{-} = \argmax_{\omega \in \Wo} \hat F^-_\Q(\omega)$.
\STATE
\STATE {\bfseries Output:} $[\hat F^-_\Q(\omega_{-}),~ \hat F^+_\Q(\omega_{+})]$. 
\end{algorithmic}
\end{algorithm}

\begin{thm}  \label{thm:main}
Let $\W$ be the unit ball of a normed function space $\H$. 
We have %
$$
\left[ \hatJlow_{\Q,\W},~~ \hatJup_{\Q,\W} \right] 
\subseteq \left [\hat F^-_{\Q} (\ratio) , ~~ \hat F^+_{\Q}(\ratio) \right],\quad \forall \ratio \in \H\,, 
$$
\bba 
\label{equ:mainVbounddelta}
\begin{split} 
\text{where ~~~~~~~~~~~}
 &%
 \hat F^+_\Q(\omega) \defeq \E_{\emp D_n^\ratio}[r] + \Wd  +  \varepsilon_n \norm{\ratio}_{\H}\,,\quad \\
 &%
 \hat F^-_\Q(\omega) \defeq \E_{\emp D_n^\ratio}[r] - \Wdneg  - \varepsilon_n \norm{\ratio}_{\H}\,.\quad%
 \end{split}
 \eea 
 Here $\neg \Q=\{-q\colon q\in \Q\}$ and hence $\Wdneg=\Wd$ if $\Q = \neg \Q$. 
 
 Further, the bound is tight, that is, $\hatJup_{\Q,\W} = \inf_{\ratio\in \H}\hat F^+_\Q(\omega)$ and $\hatJlow_{\Q,\W} = \sup_{\ratio\in\H} \hat F^-_\Q(\omega)$,   
 if $\Q$ is convex and there exists a function $q\in \Q$ that satisfies the strict feasibility condition that  $L_{\W}(q;~\emp D_n) <\varepsilon_n$.

\end{thm}
Therefore, when Eq.~\eqref{equ:tail} holds, 
 for any function set  $\funcset Q$, and any function $\ratio_+,  \ratio_-  \in \H$ 
(the choice of $\Q$, $\ratio_+$, $\ratio_-$ can depend on $\emp D_n$ arbitrarily), 
 we have %
 \begin{align} \label{equ:CL}
\prob\left (\left [\hat J^-_{\Q, *},~~  \hat J^+_{\Q, *}\right ] \subseteq 
\left [\hat F^-_{\Q}(\ratio_-), ~~ \hat F^+_{\Q}(\ratio_+)\right] \right) \geq 1-\delta\,. 
\end{align}
Theorem~\ref{thm:main} transforms the original bound in Eq.~\eqref{equ:primalbound}, framed in terms of $q$ and $L_\W(q;~\emp D_n)$, into a form that involves the density-ratio $\ratio$ and the related loss $\Wd$. 
The bounds in Eq.~\eqref{equ:mainVbounddelta} can be interpreted as assigning an error bar around the $\ratio$-based 
estimator $\hat J_{\ratio} = \E_{\emp D_n^\ratio}[r]$ in Eq.~\eqref{equ:densityraio}, with the error bar of $I_{\pm\Q}(\ratio; ~\emp D_n) + \varepsilon_n \norm{\ratio}_{\H}$. 
Specifically, the first term $I_{\pm\Q}(\ratio; ~\emp D_n)$ measures 
the discrepancy between $\emp D_n^\ratio$ and $\dist D\true$ as discussed in Eq.~\eqref{equ:sw}, whereas the second term captures the randomness in the empirical Bellman residual operator $\opt {\hat R}q\true$.

 Compared with Eq.~\eqref{equ:primalbound}, the global maximization on $q\in \Q$ is now transformed inside the $I_{\mathcal Q} (\ratio; \emp D_n)$ term, which yields a simple closed form solution when $\Q$ is a finite ball in RKHS. 
We can  optimize $\ratio_+$ and $\ratio_-$  by minimizing/maximizing $\hat F^{+}_{\Q}(\ratio)$ and $\hat F^-_\Q(\ratio)$ to obtain the tightest possible bound (and hence recover the primal bound). However, it is not necessary to find the exact globally optimal solutions for practical purpose. %
 When $\H$ is an RKHS, by the standard finite representer theorem \citep{scholkopf2018learning}, 
 the optimization on $\ratio$ reduces to a finite dimensional optimization, 
 which can be approximately solved with any practical technique without sacrificing the correctness of the bound (although the optimization quality of $\ratio$ impacts the tightness of the bound). 
 We elaborate on this in Appendix~\ref{sec:kernelw}.

\myparagraph{Length of the Confidence Interval}   
The form in Eq.~\eqref{equ:mainVbounddelta} also makes it much easier to analyze the tightness of the confidence interval. 
Suppose $\ratio = \ratio_+ = \ratio_-$ and $\Q = \neg \Q$, the length of the optimal confidence interval is 

$$
\hat J^+_{\Q,\W}-\hat J^-_{\Q,\W}
= \inf_{\ratio \in \Wo}\big \{ 2\Wd  + 2\varepsilon_n \norm{\ratio}_{\H} \big\}\,. 
$$

Given that $\varepsilon_n$ is $\bigO(n^{-1/2})$, we can make the overall length of the optimal confidence interval also $\bigO(n^{-1/2})$ if $\Wo$ is rich enough to include a \emph{good} density ratio estimator $\ratio^*$ that satisfies $I_{\Q}(\ratio^*;~\emp D_n) = \bigO(n^{-1/2})$  and has a bounded norm $\norm{\ratio^*}_{\H}$.  

Assumption~\ref{ass:data} does not ensure the existence of such $\ratio\true$. However, we can expect to have such a $\ratio\true$ if 
1) $\Q$ has an $\bigO(n^{-1/2})$ sequential Rademacher complexity \citep{rakhlin2015sequential} (which holds if $\Q$ is a finite ball in RKHS); and 2) 
$\emp D_n$ is collected following a Markov chain 
with a strong mixing condition and weakly converges to some limit distribution $\dist D_\infty$ whose support is $\X$; in this case, we can define $\ratio^*$ as the density ratio between $\dist D\true$ and $\dist D_\infty$. 
See Appendix~\ref{sec:tightness} for more discussions. 
Indeed, our experiments show that the lengths of practically constructed confidence intervals do tend to decay with an   $O(n^{-1/2})$ rate approximately. 

\paragraph{Using Data-Dependent $\W$ and $\Q$} 
 To ensure Eq.~\eqref{equ:Jstarbound} holds, 
 the choice of $\H$ cannot depend on the data $\emp D_n$, since it may introduce additional dependency and 
 hence make the concentration inequality in Theorem~\ref{thm:Lqpimain2} invalid. 
 Therefore, 
 if we want to use a data-dependent $\H$, we need to either base the construction of $\H$
 on separate holdout data, or introduce a generalization bound to account the dependence of $\H$ on the data $\emp D_n$. 
 
 Interestingly, 
 Eq.~\eqref{equ:Jstarbound}{ holds even  if 
  we  take $\Q = \Q(\emp D_n)$ 
  to be an arbitrary function of the data $\emp D_n$}. 
 This is because $\Q$ 
 is irrelevant to the concentration inequality (Eq.~\eqref{equ:tail}). 
  This justifies that one can construct  $\Q$ adaptively based on the data to get tighter confidence interval. 
   For example, we can make $\Q$ an RKHS ball centering around an estimator $\hat q \approx q\true$ given by a state-of-the-art method (e.g., fitted iteration or model-based methods). 
  A caveat is that we will need to ensure that $\prob(q_* \in  \Q(\emp D_n)) = 1$ in order to ensure  $\prob\left (J\true \in\left [\hat J^-_{\Q(\emp D_n), \set W}, ~\hat J^+_{\Q(\emp D_n), \set W}\right]\right) \geq 1-\delta$ (see Proposition~\ref{pro:Jstarbound}), which, as we discussed in  Section~\ref{sec:oracle}, can not be provably guaranteed based on only empirical observation.

\section{Experiments} \label{sec:experiments}
We compare our method with a variety of existing algorithms for obtaining asymptotic and non-asymptotic bounds
on a number of benchmarks. 
We find our method can
provide confidence interval that correctly covers the true expected reward with probability larger than the specified success probability $1-\delta$ 
(and is hence safe) across the multiple examples we tested. 
In comparison, the non-asymptotic bounds based on IS provide  much wider confidence intervals. 
On the other hand, 
the asymptotic methods, such as bootstrap,
despite giving tighter intervals, 
often fail to capture the true values
with the given probability in practice. %

\myparagraph{Environments and Dataset Construction}
We test our method on three environments: 
Inverted-Pendulum and CartPole from OpenAI Gym \citep{brockman2016openai}, and a Type-1 Diabetes medical treatment simulator.\footnote{\href{ https://github.com/jxx123/simglucose}{ https://github.com/jxx123/simglucose}.}
We follow a similar procedure as \citet{fengaccountable2020} to construct the behavior and target policies. 
more details on environments and data collection procedure are included in Appendix \ref{sec:app_exp_details}.

\myparagraph{Algorithm Settings}
We test the dual bound described in Algorithm \ref{algo:dual_bound}. 
Throughout the experiment, we always set $\W = \K$, the unit ball of the RKHS with positive definite kernel $k$, 
and set $\Q = r_\Q \tilde \K$, the ball of radius $r_\Q$ in the RKHS with another kernel $\tilde k$.  
We take both kernels to be Gaussian RBF kernel and
choose $r_\Q$ and the bandwidths of $k$ and $\tilde k$ using the procedure in Appendix~\ref{sec:sense_hypers}.
We use a fast approximation method to optimize $\ratio$ in $F_{\set Q}^+ (\ratio)$ 
and $F_{\set Q}^- (\ratio)$ 
as shown in Appendix~\ref{sec:kernelw}. 
Once $\ratio$ is found, we evaluate the bound in Eq.~\eqref{equ:mainVbounddelta} exactly 
to ensure that the theoretical guarantee holds.

\myparagraph{Baseline Algorithms} 
We compare our method with four existing baselines, 
including the IS-based non-asymptotic bound using empirical Bernstein inequality by \citet{thomas2015highope}, 
the IS-based bootstrap bound of \citet{thomas2015safe}, %
the bootstrap bound based on fitted Q evaluation (FQE) by \citet{kostrikov2020statistical}, 
and %
the bound in 
\citet{fengaccountable2020} which is equivalent to the primal bound in \eqref{equ:primalbound} but with looser concentration inequality (they use a $\varepsilon_n=O(n^{-1/4})$ threshold).

\newcommand{\delen}{.215\linewidth}
\newcommand{\degapline}{-.1\linewidth}

 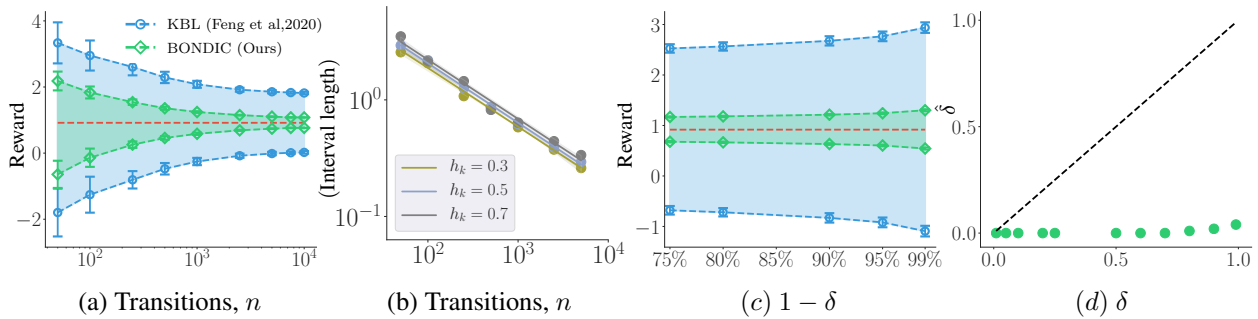
\begin{figure}[t]
    \centering
        \setlength{\tabcolsep}{0.1pt}
        \begin{tabular}{cccc}
        \vspace{-.2em}\\
        \includegraphics[height=\delen]{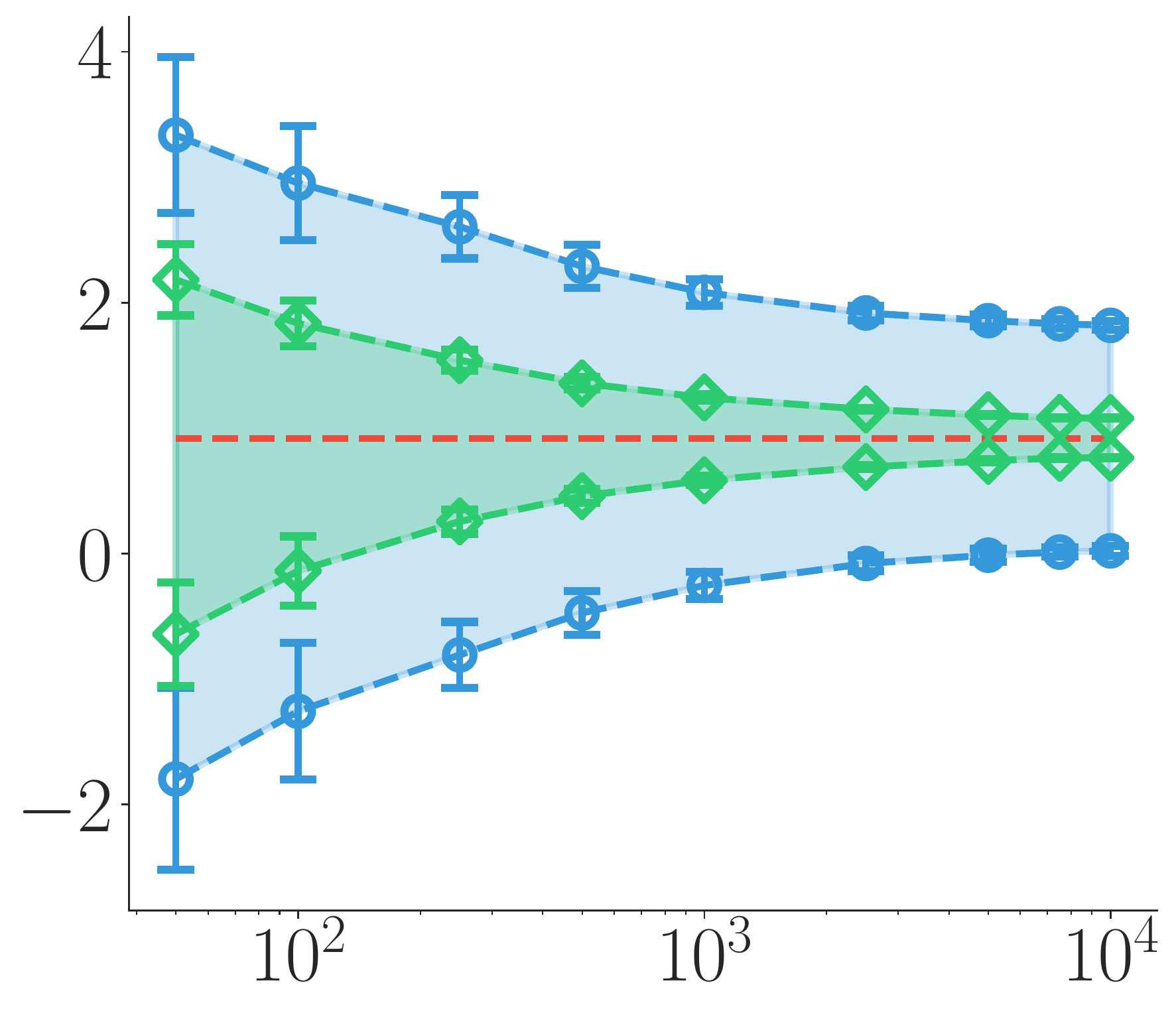}
         \llap{\makebox[\wd2][r]{\raisebox{7.3em}{\includegraphics[height=1.8em]{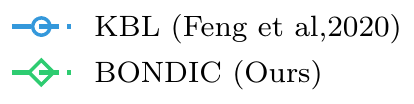}}}}&
        \includegraphics[height=.22\linewidth]{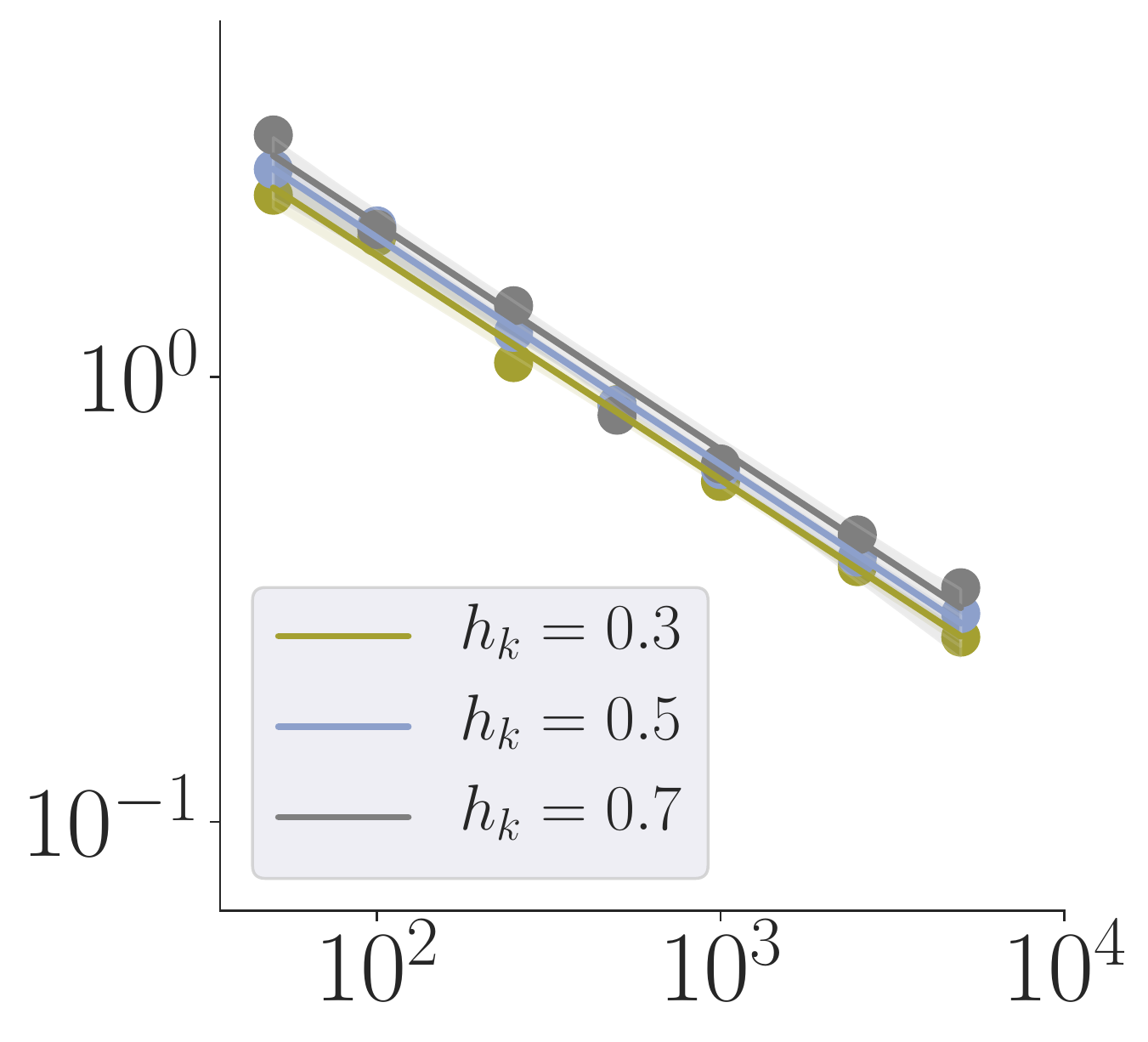} &
        \includegraphics[height=\delen]{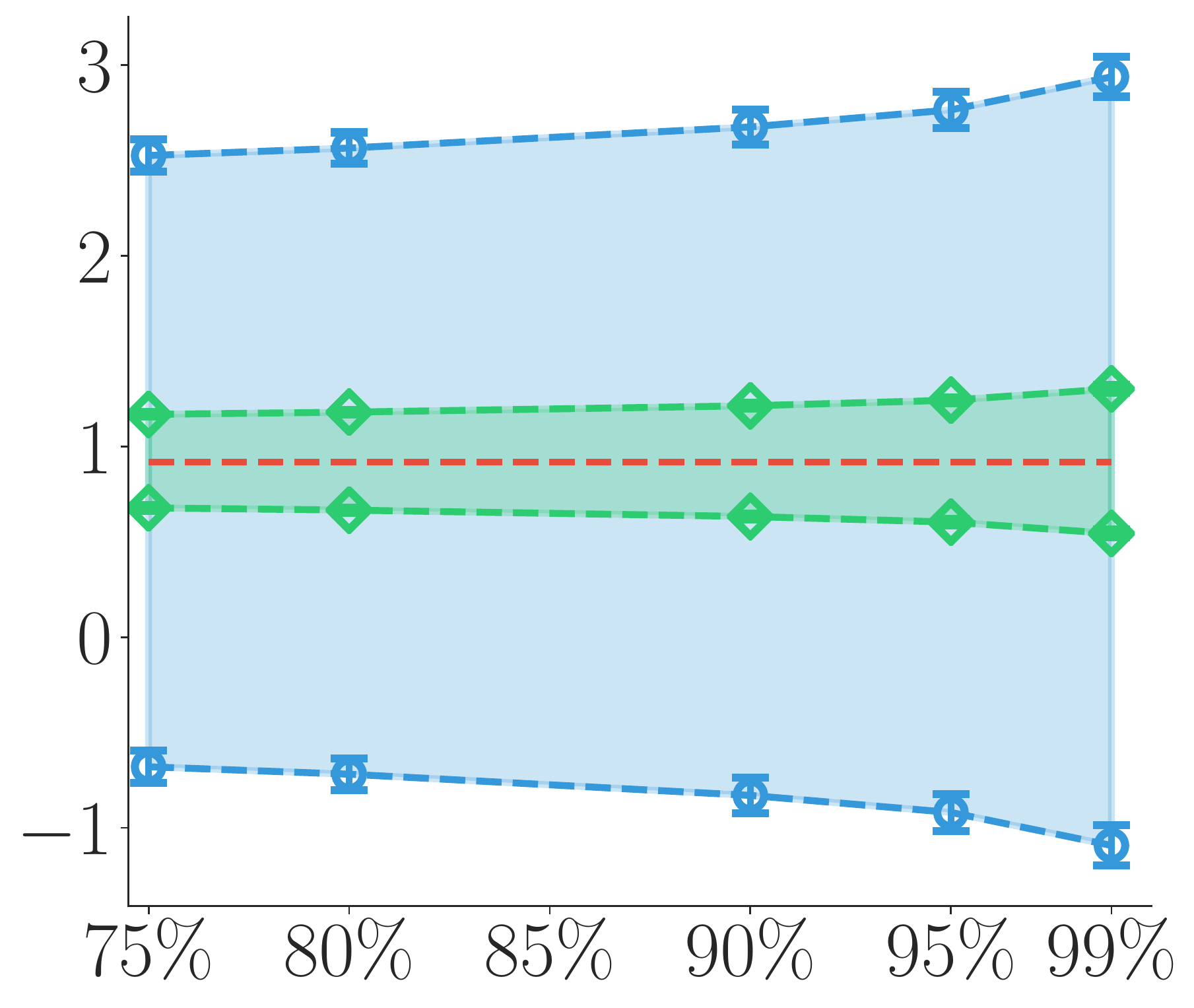} &
        \includegraphics[height=\delen]{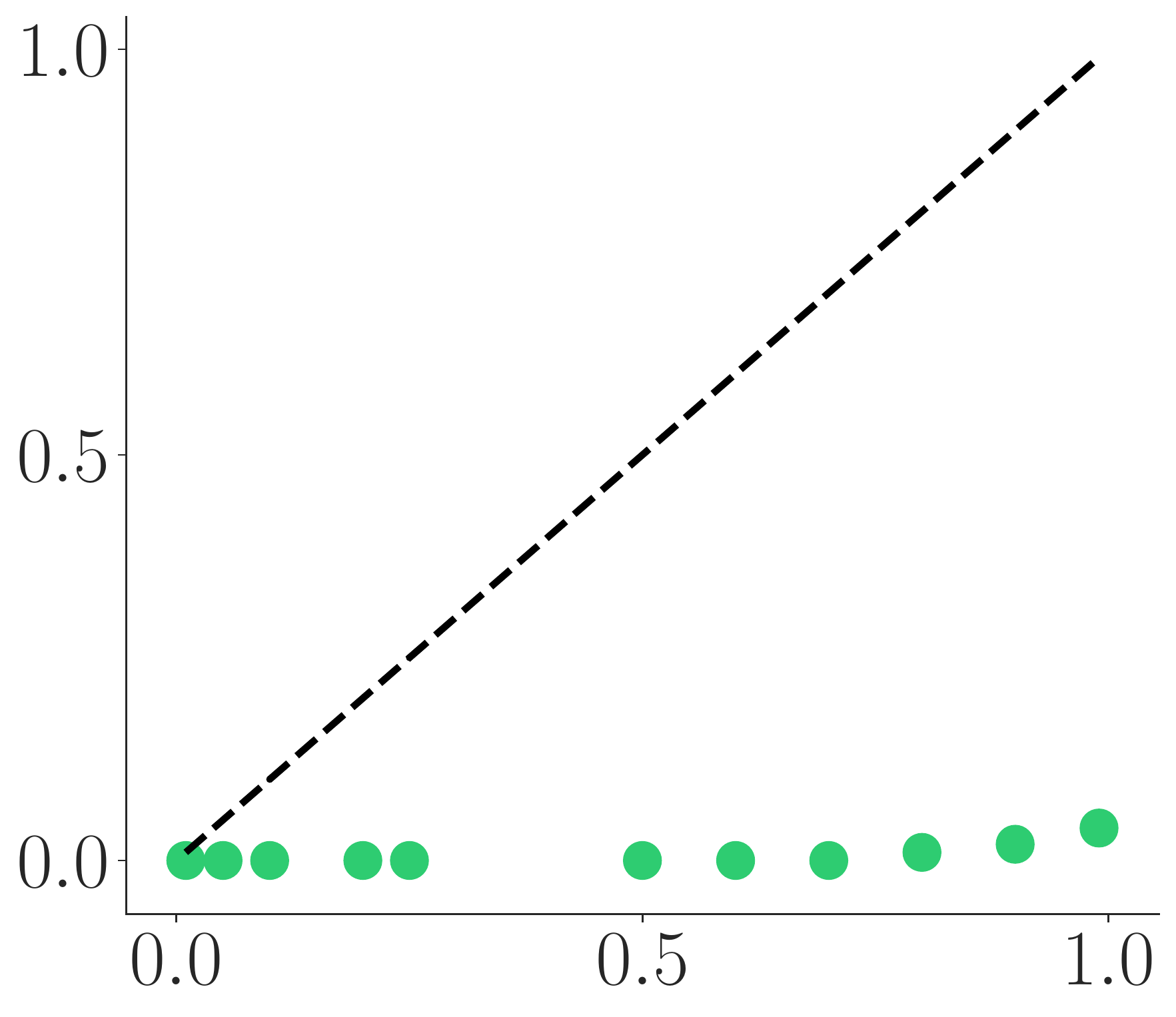}\\
        \small{(a)  Transitions, $n$} & \small{(b)~Transitions, $n$} &  \small{$(c)~1-\delta$} &  \small{$(d)~\delta$} \\
    \end{tabular}
    \begin{picture}(1,1)
        \put(-235,60){\rotatebox{90}{\scriptsize Reward}}
        \put(-118,48){\rotatebox{90}{\scriptsize $(\text{Interval length})$}}
        \put(-5,60){\rotatebox{90}{\scriptsize Reward}}
        \put(115,80){\rotatebox{90}{\scriptsize $\hat{\delta}$}}
    \end{picture}
    \caption{\small{Results on Inverted-Pendulum.  %
    (a) The confidence interval (significance level $\delta=0.1$) of our method (green) and that of \citet{fengaccountable2020} (blue) when varying the data size $n$. %
    (b) The  length of the confidence intervals ({$\delta=0.1$}) of our method scaling with the data size $n$.
    (c)  The confidence intervals when we vary the significance level $\delta$ (data size $n = 5000$). 
    (d) The significance level $\delta$ vs.  the empirical failure rate $\hat\delta$ of capturing the true expected reward by our confidence intervals (data size $n=5000$).
    We average over 50 random trials for each experiment. 
    }}
    \label{fig:inverted_pendulum}
\end{figure}

\begin{figure}[H]
    \centering
    \setlength{\tabcolsep}{0pt}
    \begin{tabular}{ccc}
        \multicolumn{3}{c}{
        \includegraphics[width=.98\linewidth]{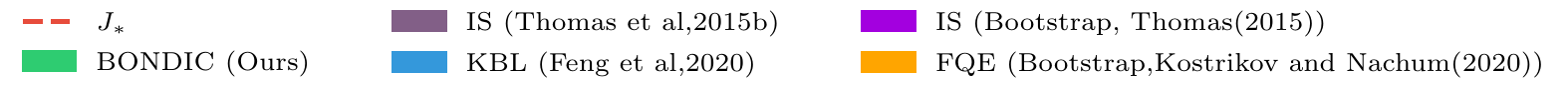}}   \\
         \includegraphics[width=.32\linewidth]{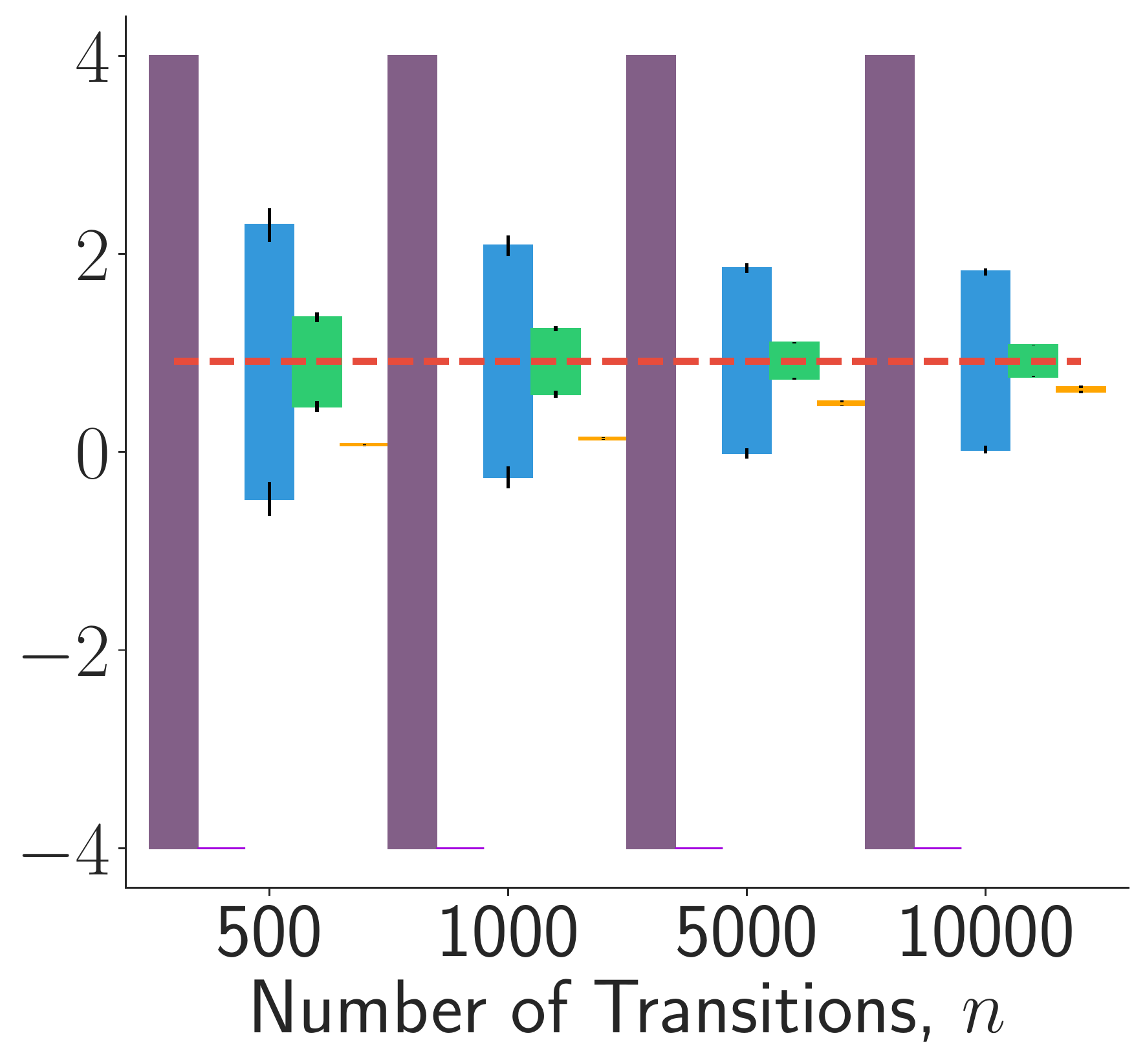} &
         \includegraphics[width=.32\linewidth]{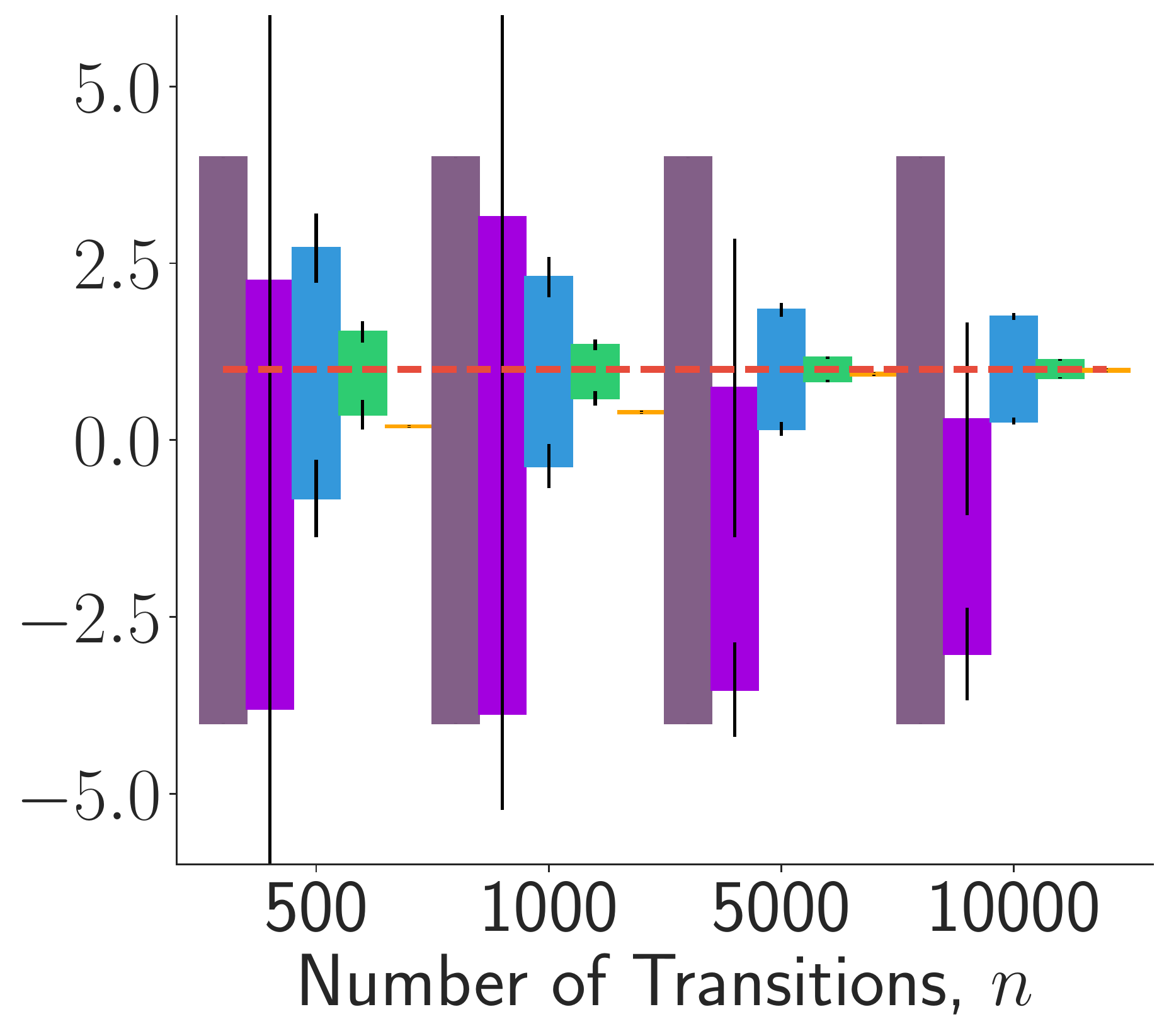} &
         \includegraphics[width=.32\linewidth]{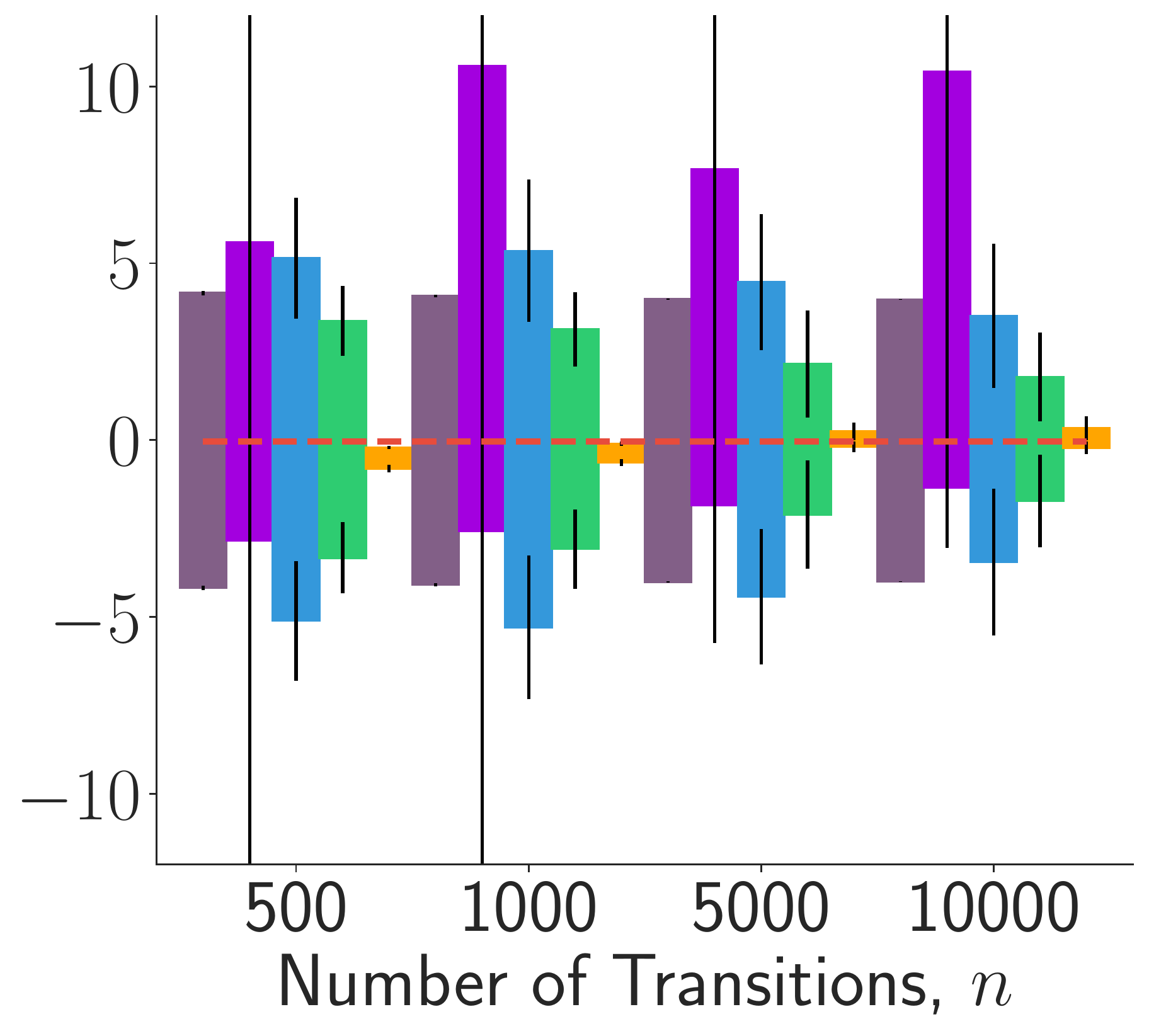}\vspace{-.3em}\\
          \small{(a)~Pendulum} & \small{(b) CartPole }& \small{(c) Type 1 diabetes} \\
    \end{tabular}
    \begin{picture}(1,1)
        \put(-470,-15){\rotatebox{90}{\footnotesize Policy Reward}}
    \end{picture}
    \caption{\small{Results on different environments when we use a significance level of $\delta=0.1$. 
    The colored bars represent the confidence intervals of different methods (averaged over 50 random trials); the black error bar represents the stand derivation of the end points of the intervals over the 50 random trials. 
    }}
    \label{fig:comparison_results}
\end{figure}
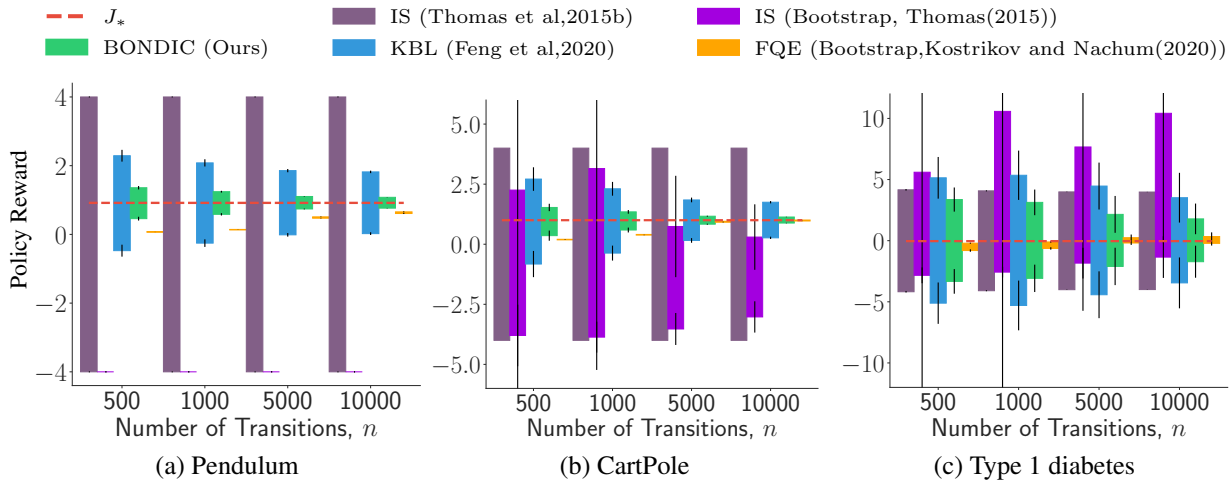

\myparagraph{Results} 
Figure~\ref{fig:inverted_pendulum} shows our method obtains much tighter bounds than \citet{fengaccountable2020}, 
which is because  we use a much tighter concentration inequality,
even the dual bound that we use can be slightly looser than the primal bound used  in  \citet{fengaccountable2020}. 
Our method is also more computationally efficient than that of \citet{fengaccountable2020} because the dual bound can be tightened approximately while the primal bound requires to solve a global optimization problem. 
Figure~\ref{fig:inverted_pendulum}~(b) shows that 
we provide increasingly tight bounds as the data size $n$ increases, and the length of the interval decays with an  $\bigO(n^{-1/2})$ rate approximately. 
Figure~\ref{fig:inverted_pendulum}~(c) shows that when we increase the significance level $\delta$, our bounds become tighter while still capturing the ground truth. 
Figure~\ref{fig:inverted_pendulum}~(d) shows the percentage of times that the interval fails to capture the true value in a total of $100$ random trials (denoted as $\hat\delta$) as we vary $\delta$.  
We can see that $\hat \delta$ remains close to zero even when $\delta$ is large, 
suggesting that our bound is very conservative. 
Part of the reason is that 
the bound is constructed by considering the worse case and we used a conservative choice of the radius $r_{\Q}$ and coefficient $c_{q_\pi, k}$ in Eq.~\eqref{equ:fastcon1} (See Appendix \ref{sec:sense_hypers}).

In Figure~\ref{fig:comparison_results} we compare different algorithms on more examples with $\delta = 0.1$.  
We can again see that our method provides tight and conservative interval that 
always captures the true value. 
Although FQE (Bootstrap) yields tighter intervals than our method, 
it fail to capture the ground truth much more often than the promised $\delta=0.1$ (e.g., it fails in all the random trials in Figure \ref{fig:comparison_results}~(a)). 

We  conduct more  ablation studies on different hyper-parameter and data collecting procedure. %
See Appendix~\ref{sec:sense_hypers} and \ref{sec:h3} for more details.

\section{Conclusion and Future Directions}\label{sec:conclusion}

We develop a practical algorithm for constructing non-asymptotic confidence intervals 
for infinite-horizon off-policy evaluation with a very mild data assumption that holds for behavior-agnostic and time-dependent data. 
Our work opens a number of future direction: %
how to apply our bounds to develop new methods for safe policy optimization and safe exploration? 
how to develop rigorous and practical approaches to select $\set Q$ and  $\set W$ adaptively based on the observed data, including their kernel, bandwidth and the radius of ball?
how to extend our method to obtain bounds for different target policies $\pi$ (see e.g., \citet{yin2020near}) and initial distributions $\dist D_0$ simultaneously and computationally efficiently, without repeatedly solving the optimization in each case? 
The oracle bound in Eq.~\eqref{equ:oracleup} 
provides %
a notion of oracle bound given the information drawn from the empirical Bellman operator. Is it possible to obtain even tighter bounds by exploiting other information from data, by e.g., model-based methods?

\bibliography{references}

\onecolumn
\appendix
\begin{center}
\Large
\textbf{Appendix}
\end{center}

\section{Proof of Lemma~\ref{lem:Mdualprimal}}
\begin{proof}%
By assumption, %
we can write any $\ratio_0\in\Wo$ into
$\ratio_0(x) = \lambda \ratio(x)$ where $\lambda \geq 0$ and $\ratio \in \set W$. Therefore, 
following Eq~\eqref{equ:jtrueminimaxM}, we have 
\bb 
\inf_{\ratiozero \in \Wo}M(q,\ratio;~ \dist D_\infty ) 
& = %
\inf_{\ratio_0 \in \Wo} 
\left \{ 
\E_{\dist D_{\pi,0}} [q] - \E_{\dist D_\infty} [\ratio_0(x)\hat\R q (x,y)] \right\} \\
& = %
\inf_{\lambda \geq 0} \inf_{\ratio \in \set W}  
\left \{
\E_{\dist D_{\pi,0}} [q] - \lambda \E_{\dist D_\infty} [\ratio(x)\hat\R q (x,y)] \right\} \\
& = %
\inf_{\lambda \geq 0} 
\left \{ 
\E_{\dist D_{\pi,0}} [q] - \lambda L_{\set W}(q, ~ \dist D_\infty) \right\},
\ee
where we used the definition that $ L_{\set W}(q, ~ \dist D_\infty)  =
\sup_{\ratio\in \set W} \E_{\dist D_\infty} [\ratio(x)\hat\R q (x,y)]$.

For Eq~\eqref{equ:MLq}, 
we note that 
\begin{align*} 
\E_{\dist D_{\pi,0}} [q(x)] - \E_{\dist D_\infty^\ratio}[\hat \R q (x,y)] 
& =
\E_{\dist D_{\pi,0}} [q(x)] - 
\E_{\dist D_\infty^\ratio}[ q(x) - \gamma q(x') - r ] \\
&=  
\E_{\dist D_\infty^\ratio}[r] + 
\Delta(\dist D_\infty^\ratio, ~q) ,
\end{align*}
where we use the definition of $\Delta(\dist D_\infty^\ratio, ~q) $ in Eq~\eqref{equ:Delta}. 
Therefore, 
\bb 
\sup_{q \in \set Q}M(q,\ratio; ~ \dist D_\infty) 
& = \sup_{q\in \set Q}
\left \{ 
\E_{\dist D_{\pi,0}} [q] -  \E_{\dist D_\infty^\ratio} [\hat\R q (x,y)] \right\}  \\
& = 
\sup_{q\in \set Q}\left \{ 
 \E_{\dist D_\infty^\ratio}[r] + \Delta(\dist D_\infty^\ratio, ~q)
\right\} \\
& = \E_{\dist D_\infty^\ratio} [r] + I_{\Q}(\ratio;~\dist D_\infty). 
\ee 
\end{proof}

\section{Proof of the Dual Bound in Theorem~\ref{thm:main}}  
\label{sec:proofdual}

 \begin{proof}
 Note that we assumed that $\set W$ is the unit ball of the normed space $\H$. %
 Therefore, we can write any $\ratio$ in $\H$ into $\ratio(x) = \lambda h(x) $ with $\lambda\in \RR$, 
 $h \in \set W$ and $\norm{\ratio}_{\H} = \lambda$.  
 
Using Lagrange multiplier, the bound in 
 Eq.~\eqref{equ:primalbound} is equivalent to 
 \begin{align*}%
\hatJup_{\Q,\W} 
& = \sup_{q \in \funcset Q}  \inf_{\lambda\geq 0} \left \{ \E_{\D_{\pi,0}}[q]  ~- ~\lambda\left (\sup_{h\in \W} \frac{1}{n}\sum_{i=1}^n h(x_i) \hat\R q(x_i, y_i)  -  \varepsilon_n \right ) \right \}  \\ 
& = \sup_{q \in \funcset Q}  \inf_{\lambda\geq 0} \inf_{h\in \W} \left \{ \E_{\D_{\pi,0}}[q]  ~- ~\lambda\left ( \frac{1}{n}\sum_{i=1}^n h(x_i) \hat\R q(x_i, y_i)  -  \varepsilon_n \right ) \right \} \\
& = \sup_{q \in \funcset Q} \inf_{\ratio \in {\H} } \left \{ \E_{\D_{\pi,0}}[q]  ~- ~ \frac{1}{n}\sum_{i=1}^n \ratio(x_i) \hat\R q(x_i, y_i)  +  \varepsilon_n\norm{\ratio}_{\H}  \right \}\,, 
\end{align*}
 Define 
 \bb 
 \hat M(q, ~ \ratio; ~ \emp D_n)
 &= \E_{\D_{\pi,0}}[q]  ~-  \frac{1}{n}\sum_{i=1}^n \ratio(x_i) \hat\R q(x_i, y_i)  +  \varepsilon_n  \norm{\ratio}_{\H} \\
  & = 
  \E_{\emp D_n^\ratio} [r]  + \Delta(\emp D_n^\ratio,~ q) +  \varepsilon_n  \norm{\ratio}_{\H}. 
 \ee 
 Then we have 
 \bb 
 \sup_{q \in \funcset Q} \hat M(q, ~ \ratio; ~ \emp D_n) 
 & =  \E_{\emp D_n^\ratio} [r]  + \sup_{q \in \funcset Q} \Delta(\emp D_n^\ratio,~ q) +  \varepsilon_n  \norm{\ratio}_{\H} \\
  & =  \E_{\emp D_n^\ratio} [r]  + \Wd +  \varepsilon_n  \norm{\ratio}_{\H} \\
  & := \hat F^+_{\Q}(\ratio). 
 \ee 
 Therefore, 
 \bb 
 \hatJup_{\Q,\W} 
 & =   \sup_{q \in \funcset Q} \inf_{\ratio \in \H} M(q, ~ \ratio; ~ \emp D_n) \\
  & \leq    \inf_{\ratio \in \H}   \sup_{q \in \funcset Q}  M(q, ~ \ratio; ~ \emp D_n) \\
  & =\inf_{\ratio \in \H}  \hat F^+_{\Q}(\ratio).   
 \ee 
 The lower bound follows analogously.  
 The strong duality holds when the Slater's condition \citep{nesterov2003introductory} is satisfied, 
 which amounts to saying that the primal problem in Eq.~\eqref{equ:primalbound} 
 is convex and strictly feasible; this requires that $\Q$ is convex and there exists at least one solution $q\in \Q$  that satisfies that constraint strictly, that is,  $L_\W(q;~\emp D_n) < \varepsilon_n$; 
 note that in our case the objective function $\Q$ is linear on $q$ and the constraint function $L_\W(q;~\emp D_n)$ is always convex on $q$ (since it is the sup a set of linear functions on $q$ following Eq.~\eqref{equ:kblfunctional}).  
 
 \end{proof}

\section{Proof of Concentration Bound in Theorem~\ref{thm:Lqpimain2}}
\label{sec:proofcon}

Our proof requires the following Hoeffding inequality on Hilbert spaces by 
\citet[Theorem 3,][]{pinelis1992approach};  see also Section 2.4 of \citet{rosasco2010learning}. %
\begin{lem} \label{pinelis}
\citep[Theorem 3,][]{pinelis1992approach}
Let $\funcset H$ be a Hilbert space and $\{f_i\}_{i=1}^n$ is a martingale difference sequence in $\funcset H$ that satisfies $\sup_i\norm{f_i}_{\funcset H} \leq \sigma$ almost surely. We have for any $\epsilon >0$, 
$$
\prob\left (\norm{ \frac{1}{n} \sum_{i=1}^n f_i}_{\funcset H} \geq \epsilon  \right ) \leq 2 \exp\left (- \frac{n\epsilon^2}{2\sigma^2} \right ). 
$$
Therefore, for $\delta\in (0,1)$, with probability at least $1-\delta$, we have 
$\norm{ \frac{1}{n} \sum_{i=1}^n f_i}_{\funcset H} \leq  \sqrt{\frac{2\sigma^2\log(2/\delta)}{n}}.$
\end{lem}

\begin{lem}\label{lem:Kffdji} 
Let $k(x,x')$ be a positive definite kernel whose RKHS is $\funcset H_{k}$. 
Define 
$$f_i(\cdot) = \hat{\opt R} q(x_i, y_i) k(x_i, \cdot) -  \Rtrue q(x_i) k(x_i, \cdot).$$
Assume Assumption~\ref{ass:data} holds, 
then $\{f_i\}_{i=1}^n$ is a martingale difference sequence in $\funcset H_{k}$ w.r.t. 
$T_{<i}\defeq (x_{j}, y_j)_{j<i}\cup (x_i)$.
That is, 
$\E\left [f_{i+1}(\cdot) ~|~ T_{<i}\right ] = 0$ for $i=1,\ldots,n$.  
In addition, 
$$
\norm{\frac{1}{n}\sum_{i=1}^n f_i}^2_{\funcset H_{k}} 
=  \frac{1}{n^2}\sum_{ij=1}^n \left (\hat{\opt R} q(x_i, y_i) -  \Rtrue q(x_i)\right ) k(x_i, x_j)
\left (\hat{\opt R} q(x_j, y_j)  -  \Rtrue q(x_j) \right), 
$$
and $\norm{f_i}_{\funcset H_k}^2 \leq c_{q, k}$ for $\forall i = 1,
\ldots, n$. 
\end{lem}

\begin{proof}[Proof of Theorem~\ref{thm:Lqpimain2}]
Following Lemma~\ref{pinelis} and Lemma~\ref{lem:Kffdji},  
since $\{f_i\}_{i=1}^n$ is a martingale difference sequence in $\funcset H_k$ with $\norm{f_i}_{\funcset H_k}^2\leq c_{q, k}$ almost surely, we have with probability at least $1-\delta$, 
$$
\frac{1}{n^2}\sum_{ij=1}^n\left (\hat{\opt R} q(x_i, y_i) -  \Rtrue q(x_i)\right ) k(x_i, x_j)
\left (\hat{\opt R} q(x_j, y_j)  -  \Rtrue q(x_j) \right) = 
\norm{ \frac{1}{n} \sum_{i=1}^n f_i}_{\funcset H_k}^2 \leq  \frac{2c_{q,k} \log(2/\delta)}{n}.$$
Using {Lemma~\ref{lem:tria2}} below, we have 
$$
 \abs{ { L_{\K}(q;~ \emp D_n)} -  { L^*_\K(q;~ \emp D_n)} } 
\leq  
\norm{ \frac{1}{n} \sum_{i=1}^n f_i}_{\funcset H_k} \leq  \sqrt{\frac{2c_{q,k} \log(2/\delta)}{n}}.
$$
This completes the proof. 
\end{proof} 

\begin{lem}\label{lem:tria2} 
Assume $k(x,x')$ is a positive definite kernel.  We have 
$$
\abs{ { L_\K(q;~ \emp D_n)} -  { L^*_\K(q;~ \emp D_n)} }^2 
\leq 
\frac{1}{n^2}\sum_{ij=1}^n \left (\hat{\opt R} q(x_i, y_i) -  \Rtrue q(x_i)\right ) k(x_i, x_j)
\left (\hat{\opt R} q(x_j, y_j)  -  \Rtrue q(x_j) \right). 
$$
\end{lem}

\begin{proof} Define 
\begin{align*}
 \hat g(\cdot) = \frac{1}{n}\sum_{i=1}^n \hat{\opt R} q(x_i, y_i) k(x_i, \cdot),  
&&  g(\cdot) = \frac{1}{n}\sum_{i=1}^n  \Rtrue q(x_i) k(x_i, \cdot) .
\end{align*}
Then we have 
\bb 
&\norm{\hat g}_{\funcset H_k}^2 = \frac{1}{n^2}\sum_{ij=1}^n \hat{\opt R} q(x_i, y_i) k(x_i, x_j)  \hat{\opt R} q(x_j, y_j) = \hat L_\K(q;~ \emp D_n)^2,   \\ 
&\norm{ g}_{\funcset H_k}^2 = \frac{1}{n^2}\sum_{ij=1}^n \Rtrue  q(x_i) k(x_i, x_j)  \Rtrue  q(x_j)  = L^*_\K(q;~ \emp D_n)^2,  \\ 
&\norm{ \hat g  - g}_{\funcset H_k}^2 = \frac{1}{n^2}\sum_{ij=1}^n\left (\hat{\opt R} q(x_i, y_i) -  \Rtrue q(x_i)\right ) k(x_i, x_j)
\left (\hat{\opt R} q(x_j, y_j)  -  \Rtrue q(x_j) \right). 
\ee 
The result then follows the triangle inequality 
$ 
\abs{\norm{\hat g}_{\funcset H_k} - 
\norm{ g}_{\funcset H_k} } \leq \norm{ \hat g  - g}_{\funcset H_k}. 
$ 
\end{proof}

\subsection{{Calculation of \texorpdfstring{$c_{q\true,k}$}{calculation} }}\label{sec:calcq}
The practical calculation of the coefficient $c_{q\true,k}$ in the concentration inequality was discussed in \citet{fengaccountable2020}, which we include here for completeness. 
\begin{lem}
{
(\citet{fengaccountable2020} Lemma~3.1) Assume the reward function and kernel function is bounded with $\sup_{x}|r(x)| \leq r_{\max}$ and $\sup_{x,x'}|k(x,x')|\leq K_{\max}$, we have: 
\begin{align}
    c_{q\true,k} := \sup_{x \in \set{X},y \in \set{Y}} ({\hat \R} q\true(x,y))^2 k(x,x) \leq \frac{4 K_{\max}r^2_{\max}}{(1-\gamma)^2}\,. \label{equ:c_qk}
\end{align}

}
\end{lem}
{
In practice, we evaluate $K_{\max}$ from the kernel function that we choose (e.g., $K_{\max} = 1$ for RBF kernels), and $r_{\max}$ from the knowledge of the environment.} 

\section{More Discussion on the Tightness of the Confidence Interval} %
\label{sec:tightness} 
The benefit of having both upper and lower bounds is that we can empirically access the tightness 
of the bound by checking the length of the interval $[\hat F^-_{\Q}(\ratio_-), \hat F^+_\Q (\ratio_+)]$. 
However, from the theoretical perspective, 
it is desirable to know \emph{a priori} that the length of the interval will decrease with a fast rate as the data size $n$ increases. 
We now show that this is the case  
if $\Wo$ is chosen to be sufficiently rich so that it includes a $\ratio \in \Wo$ such that $\emp D_n^\ratio$ approximates  $\dist D\true$  closely in a proper sense.  

\begin{thm}%
\label{thm:uppergap}
Assume $\Wo$ is sufficiently rich to  
include a ``good'' $\ratio\true$ in $\Wo$ with $\emp D_n^{\ratio\true} \approx \dist D\true$ in  the sense that 
  \begin{align} \label{equ:assbig}
  \sup_{q\in \funcset Q}
\abs{ \E_{%
\emp  d_n^{\ratio\true}}
\left [  {\hat \R} q(x,y) \right ]
- \E_{%
\dist d\true}  
\left [ {\hat \R} q(x,y) \right ]} \leq \frac{c}{n^\alpha},
\end{align} 
where $c$ and $\alpha$ are two positive coefficients. 
Then we have 
\begin{align*} 
\max \left \{
 \hatJup_{\Q,\W}
- \Jpi,~~ 
\Jpi  - \hatJlow_{\Q,\W}
\right \} 
\leq 
\frac{c}{n^\alpha} + 
\varepsilon_n \norm{\ratio\true}_{\Wo}. 
\end{align*}
\end{thm}
Assumption \eqref{equ:assbig} %
 holds if $\emp D_n$ is collected following a Markov chain that has certain strong mixing condition and weakly converges to a limit continuous discussion $\dist D_\infty$ whose support is $\X$, 
 and 
 the density ratio between $\dist D\true$ and $\dist D_\infty$, denoted by $\ratio\true$, 
 is included in $\Wo$. 
 In this case, if $\Q$ is a finite ball in RKHS, then we can achieve Eq.~\eqref{equ:assbig} with $\alpha=1/2$, and yields the overall bound of rate $O(n^{-1/2})$. For more general function classes, $\alpha$ depends on the martingale Rademacher complexity of the function set $\hat\R \Q =\{\R q(x,y) \colon q\in \Q \}$ \citep{rakhlin2015sequential}.   
 In our empirical reults, we observe that the lengths of the practically constructed confidence intervals 
 do tend to follow the  $O(n^{-1/2})$ rate approximately; see Figure~\ref{fig:inverted_pendulum}(b)

\begin{proof}
Note that $\gamma q(x') - q(x) = - \hat\R q(x,y) - r$ and hene 
\bb 
\Wd & = \sup_{q\in \Q} \left \{ \E_{\emp D_n^\ratio} [\gamma q(x') - q(x)]  - \E_{\dist D\true} [\gamma q(x') - q(x)]  \right\} \\ 
& = \sup_{q\in \Q} \left \{ 
 \E_{\dist D\true} [\hat\R q(x,y)  ]  - 
\E_{\emp D_n^\ratio} [ \hat\R q(x,y)  ] 
 \right\}
  + \E_{\dist D\true} [  r ]  - 
\E_{\emp D_n^\ratio} [  r ] . 
\ee 
Because $\ratio\true \in \W$, we have 
\bb 
\hatJup_{\W,\Q} 
- \Jpi
& \leq \hat F^+_\Q(\ratio\true)
- \Jpi \\
& = %
\E_{\emp D_n^\ratio}[r] 
+ \Wdpi  + \varepsilon_n \norm{\ratio\true}_{\Wo} -  \E_{\dist d\true }[r]  \\
& = 
\sup_{q\in \funcset Q}
\left \{ 
 \E_{ \dist d\true } \left [{\hat \R} q(x,y) \right ] 
- \E_{\emp  D_n^\ratio} \left [{\hat \R} q(x, y) \right ]  
 \right\} 
+ \varepsilon_n \norm{\ratio\true}_{\Wo} \\
& \leq \frac{c}{n^\alpha} +  \varepsilon_n \norm{\ratio\true}_{\Wo}. 
\ee 
The case of lower bound follows similarly. 
\end{proof}

\section{Practical Optimization on \texorpdfstring{$\H$}{H} }\label{sec:kernelw}
Consider the optimization of $\ratio$ in $\H$, 
\bba \label{equ:fwhahaha} 
\hat F^+_{\Q}(\omega) \defeq 
\frac{1}{n}\sum_{i=1}^n {r_i \ratio(x_i)} +  
\Wd  +  \norm{\ratio}_{\H}
\sqrt{\frac{{2}c_{q\true,k} \log (2/\delta)}{n}}. 
\eea 
Assume $\Wo$ is the RKHS of kernel $k(x,\bar x)$. %
 By the finite representer theorem of RKHS \citep{scholkopf2018learning},  %
 the optimization of $\omega \in \Wo$ can be reduced to a finite dimensional optimization problem. Specifically, 
 the optimal solution of \eqref{equ:fwhahaha} 
 can be written into a form of $\ratio(x) = \sum_{i=1}^n k(x, x_i) \alpha_i$  
 for which we have 
 $\norm{\ratio}_{\funcset H_k}^2 = \sum_{i,j=1}^n k(x_i, x_j) \alpha_i \alpha_j$ 
 for some vector $\vv\alpha\defeq [\alpha_i]_{i=1}^n  \in \RR^n$. 
 Write $\vv K=[k(x_i,x_j)]_{i,j=1}^n$ and $\vv r = [r_i]_{i=1}^n$. 
 The optimization of $\ratio$ reduces to solving the following optimization on $\vv\alpha$: 
 $$
 \min_{\vv\alpha\in \RR^n}\left\{
\frac{1}{n} \vv r^\top \vv K\vv \alpha + \WdKa + \sqrt{\vv \alpha \vv K \vv \alpha } \sqrt{\frac{{2}c_{q\true,k} \log (2/\delta)}{n}} \right\}\,, 
 $$
 where 
 \bb  %
 \WdKa =  
 \max_{q\in \funcset Q} \left \{\E_{\D_{\pi,0}}[q] +  
 \frac{1}{n} (\hat{\opt T} q)^\top \vv K \vv \alpha    \right\}, 
 \ee  
 and $\hat{\opt T} q
 = [\gamma q(x'_i)- q(x_i)]_{i=1}^n \in \RR^n$. 
 When $\funcset Q$ is a finite ball of an RKHS (that is different from $\Wo$), %
 we can calculate  $\WdKa$ using Eq.~\eqref{equ:Wfkernel}. %

 This computation  can be still expensive when $n$ is large. 
 Fortunately, our confidence bound holds correctly for any $\ratio$;  
 better $\ratio$ only gives tighter bounds, 
 but it is not necessary to find the exact global optimal $\ratio$. 
 Therefore, one can use any approximation algorithm to  find $\ratio$, which provides a trade-off of tightness and  computational cost. 
 We discuss two methods:

 \paragraph{1) Approximating $\vv\alpha$}  The length of $\vv \alpha$ can be too large when $n$ is large. 
 To address this, we assume $\alpha_i = g(x_i, ~ \theta)$, where $g$ is any parametric function (such as a neural network) with  a parameter $\theta$; assume $\theta$ has  a much lower dimension than $\vv\alpha$. 
 We can then optimize $\theta$ with stochastic gradient descent, 
 by approximating the empirical averaging $\frac{1}{n}\sum_{i=1}^n(\cdot)$ in the objective with averages over small mini-batches; this would introduce biases in gradient estimation, 
 but it is not an issue when the goal is only to get a reasonable approximation of the optimal $\ratio$.  %

\paragraph{2)  Replacing kernel $\vv k$}
Assume the kernel $k(x, \bar x)$ yields a random feature expansion of form $k(x,\bar x) = \E_{\beta \sim \pi}[\phi(x, \beta) \phi(\bar x, \beta)]$, where $\phi(x,\beta)$ is a feature map with parameter $\beta$ and $\pi$ is a distribution of $\beta$. 
We draw $\{\beta_i\}_{i=1}^m$ i.i.d. from $\pi$,  where $m$ is taken to be much smaller than $n$.  
 We replace $k$ with $\hat k(x,\bar x) =\frac{1}{m}\sum_{i=1}^m \phi(x, \beta_i)\phi(\bar x, \beta_i)$ and let $\widehat{\set W}$ to be the RKHS of kernel $\hat k$. 
 Then, we consider to solve 
 $$
\hatJup_{\Q,\W} = \min_{\ratio\in \widehat{\set W} 
}
\left\{  \hat F^+_\Q(\omega) \defeq 
\frac{1}{n}\sum_{i=1}^n {r_i \ratio(x_i)} +  
\Wd  +  \norm{\ratio}_{\widehat{\set W}}
\sqrt{\frac{{2}c_{q\true,\hat k} \log (2/\delta)}{n}} \right\}.  
$$
It is known that 
any function $\ratio$ in $\widehat{\set W}$ can be represented as 
$
\ratio (x) = \frac{1}{m}\sum_{i=1}^m w_i \phi(x, \beta_i), 
$ 
for some $\vv w = [w_i]_{i=1}^m\in \RR^m$  
and satisfies  $\norm{\ratio}_{\widehat{\set W}}^2 =  \frac{1}{m}\sum_{i=1}^m w_i^2.$ 
In this way, the problem reduces to optimizing an $m$-dimensional vector $\vv w$, 
which can be (approximately) solved by standard convex optimization techniques.

\section{Concentration Inequality of General Functional Bellman Losses} 
\label{sec:proofrade}
Theorem~\ref{thm:Lqpimain2} 
provides the concentration inequality of $L_{\W}(q;~\emp D_n)$ when $\W$ is  the unit ball of an RKHS. 
When $\W$ is a general function set, 
one can still obtain a general  
concentration bound using Radermacher complexity. 
Define  $\hat\R q \circ \W\defeq \{h(x,y)=\hat \R q(x,y) \ratio(x) \colon \ratio\in \W\}$. Using  
the standard derivation in  Radermacher complexity theory in conjunction with Martingale theory \citep{rakhlin2015sequential}, we have 
\begin{align*} 
\sup_{\ratio \in \W} \left\{\frac{1}{n}\sum_{i=1}^n  (\hat \R q(x_i,y_i) -\Rtrue q(x_i))\ratio(x_i) \right\}
\leq 
2 Rad(\hat\R q \circ \W) + \sqrt{\frac{2 c_{q,\W}\log(2/\delta)}{n}},   
\end{align*}
where 
 $c_{q,\W} = \sup_{\ratio\in \W}\sup_{x,y} (\hat\R q(x,y)-\R q(x))^2 \ratio(x)^2$ 
 and 
$Rad(\hat\R q \circ \W)$ is the sequential 
 Radermacher complexity \citep{rakhlin2015sequential}. 
 A triangle inequality yields 
 $$
 |~L_\W(q;~\emp D_n) - L^{*}_\W(q;~\emp D_n)  ~|\leq 
\sup_{\ratio \in \W} \left\{\frac{1}{n}\sum_{i=1}^n  (\hat \R q(x_i,y_i) -\Rtrue q(x_i))\ratio(x_i) \right\}. 
 $$
 Therefore, 
\begin{align} \label{equ:rade}
|~L_\W(q;~\emp D_n) - L^{*}_\W(q;~\emp D_n)  ~| \leq 
2 Rad(\hat\R q \circ \W) + \sqrt{\frac{2 c_{q,\W}\log(2/\delta)}{n}}.   
\end{align}
When $\W$ equals the unit ball $\K$ of the RKHS of kernel $k$, we have $c_{q,k} = c_{q, \W}$ and hence  this bound is strictly worse than the bound in Theorem~\ref{thm:Lqpimain2}.

\section{More on the Oracle Bound and its Dual Form} 
\label{sec:oracle_appendix}

The oracle bound \eqref{equ:oracleup} provides another starting point for deriving optimization-based confidence bounds. 
We start with deriving the dual form of \eqref{equ:oracleup}. 
Using 
Lagrangian multiplier, the optimization in Eq.~\eqref{equ:oracleup} can be rewritten into 
\bba \label{equ:lagfunhat} 
\hat J_{\funcset Q, *}^+ =
\sup_{q\in \funcset Q} \inf_{\ratio}  M_{*}(q, \ratio; ~ \emp D_n) 
\leq \inf_{\ratio} \sup_{q\in \funcset Q}   M_{*}(q, \ratio; ~ \emp D_n), 
\eea 
where 
$$
 M_*(q, \ratio; ~ \emp D_n) =  \E_{\D_{\pi,0}}[q] - \frac{1}{n}\sum_{i=1}^n 
\ratio(x_i) \left ( \hat\R  q(x_i,y_i) - \hat\R q\true(x_i,y_i)\right),
$$
and $\ratio$ is optimized in the set of all functions and serves as the Lagrangian multiplier here.  
Define 
 \begin{align*} 
 \hat F_{\funcset Q,*}^+(\ratio)  &\defeq   
 \max_{q\in \funcset Q} M_*(q, \ratio; ~ \emp D_n)  \\
  & = \underbrace{
 \E_{\emp D_n^\ratio} [r] 
 + \Wd}_{known} + \underbrace{
 R(\ratio, ~ q\true) 
 }_{unknown} 
\end{align*}
where 
$$R(\ratio, q\true) =  \frac{1}{n}\sum_{i=1}^n \ratio(x_i) \hat\R q\true(x_i).$$ 
Then by the weak duality, we have 
$$
J_{\funcset Q, +}^* \leq  \hat F_{\funcset Q,*}^+(\ratio),  ~~~~\forall \ratio. 
$$
The derivation follows similarly for the lower bound. So for any $\ratio \in \Wo$, 
we  have  
$
[\hat J_{\funcset Q,*}^-, ~~
\hat J_{\funcset Q,*}^+  ]\subseteq 
[\hat F_{\funcset Q,*}^-(\ratio) , ~~
\hat F_{\funcset Q,*}^+(\ratio)  ].$ 

Here the first two terms of 
$ \hat F_{\funcset Q,*}^+(\ratio)$ can be empirically estimated (it is the same as the first two terms of Eq.~\eqref{equ:mainVbounddelta}), but the third term $R(\ratio, q\true)$ depends on the unknown $q\true$ and hence need to be further upper bounded. 
Different upper bounds of $R(\ratio, q\true)$  may yield different practical confidence intervals. %

Our method can be viewed as constraining $\ratio$ in the unit ball $\W$ of a normed function space $\Wo$, and applying a worst case bound:  for any $\ratio\in \Wo$, we have 
\bb 
 \hat F_{\funcset Q,*}^+(\ratio)  
 & \defeq \E_{\emp D_n^\ratio} [r] + \Wd + R(\ratio, ~ q\true) \\ %
  & \leq  \E_{\emp D_n^\ratio} [r] + \Wd + \norm{w}_{\Wo} \sup_{h\in \W } R(h, ~ q\true)  \\
    & \leq  \E_{\emp D_n^\ratio} [r] + \Wd +   \norm{w}_{\Wo} L_{\W}(q\true, \emp D_n)  \\
    & \overset{w.p. 1-\delta}{\leq}  \E_{\emp D_n^\ratio} [r] + \Wd +  \varepsilon_n \norm{w}_{\Wo}  \\
    & = \hat F_{\funcset Q}^+(\ratio),
\ee 
where we note that $L_{\W}(q\true, \emp D_n) = \sup_{h\in \W } R(h, ~ q\true) $ and  
the last step applies the high probability bound that $\prob( L_{\W}(q\true, \emp D_n) \leq \varepsilon ) \geq 1-\delta$. 
With the same  derivation on the lower bound counterpart, we have %
$$
\prob\left ( 
\left [\hat F_{\funcset Q,*}^-(\ratio) , ~~
\hat F_{\funcset Q,*}^+(\ratio) \right ]\subseteq \left [  \hat F_{\funcset Q}^-(\ratio) , \hat F_{\funcset Q}^+(\ratio) 
\right ]
\right ) \geq 1-\delta. 
$$
Therefore, our confidence bound $[  \hat F_{\funcset Q}^-(\ratio) , \hat F_{\funcset Q}^+(\ratio) ]$ 
is a $1-\delta$ confidence outer bound %
the oracle bound $
[\hat J_{\funcset Q,*}^- (\ratio), ~~
\hat J_{\funcset Q,*}^+ (\ratio)  ]\subseteq 
[\hat F_{\funcset Q,*}^-(\ratio) , ~~
\hat F_{\funcset Q,*}^+(\ratio)  ]$.

\subsection{Proof of Proposition ~\ref{pro:free}}

\begin{proof}

Let $\funcset Q_{\mathrm{null}}$ be the set of functions that are zero on  $\{s_i, s_i'\}_{i=1}^n$, that is, 
$$
\funcset Q_{\mathrm{null}} = 
\{ g\colon \set S \times \set A \to \RR \colon ~~ g(s,a) = 0,~~\forall s \in \{s_i, s_i'\}_{i=1}^n,~~~ a \in \set A\}. 
$$
Then we have 
$$
\hat\Rtrue (q\true + g)(x_i, y_i) = \hat\Rtrue q\true(x_i,y_i),~~~~~\forall i = 1,\ldots, n.
$$
and 
$$
\E_{\D_{\pi,0}}[q\true +g] = \E_{\D\true,0}[q\true] +  \E_{\D_{\pi,0}}[g] = J\true + \E_{\D_{\pi,0}}[g] .
$$
Taking $g(s,a) = z\ind(s \notin \{s_i, s_i'\}_{i=1}^n)$, where $z$ is any real number.  
Then we have 
$$
\E_{\D_{\pi,0}}[q\true +g] = J\true + z \prob_{s\sim \dist d_{\pi,0}}(s\notin \{s_i, s_i'\}_{i=1}^n). 
$$
Because $\prob_{s\sim \dist d_{\pi,0}}(s\notin \{s_i, s_i'\}_{i=1}^n) \neq 0$, we can take $z$ to be arbitrary value to make $\E_{\D_{\pi,0}}[q\true +g]$ to take arbitrary value. 
\end{proof}

\section{Ablation Study and Experimental Details}\label{sec:app_exp}

\subsection{Experimental Details}\label{sec:app_exp_details}
\myparagraph{Environments and Dataset Construction}
We test our method on three environments: 
Inverted-Pendulum and CartPole from OpenAI Gym \citep{brockman2016openai}, and a Type-1 Diabetes medical treatment simulator.
For Inverted-Pendulum we discretize the action space to be $\{-1, -0.3, -0.2, 0, 0.2, 0.3, 1\}$.
The action space of CartPole and the medical treatment simulator are both discrete.

\paragraph{Policy Construction} 
We follow a similar setup as \citet{fengaccountable2020} to construct behavior and target policies.
For all of the environments, we constraint our policy class to be a softmax policy and use PPO~\citep{schulman2017proximal} to train a good policy $\pi$, 
and we use different temperatures of the softmax policy to construct the target and behavior policies (we set the temperature $\tau = 0.1$ for target policy and $\tau = 1$ to get the behavior policy, and in this way the target policy is more deterministic than the behavior policy). 
We consider other choices of behavior policies in Section~\ref{sec:h3}. 

For horizon lengths, We fix $\gamma = 0.95$ and set horizon length $H = 50$ for Inverted-Pendulum, $H=100$ for CartPole, and $H=50$ for Diabetes simulator.

\myparagraph{Algorithm Settings}
We test the bound in Eq.\eqref{equ:mainVbounddelta}-\eqref{equ:CL}.
Throughout the experiment, we always set $\W = \K$, a unit ball of RKHS with kernel $k(\cdot,\cdot)$. We set $\Q = r_\Q \tilde \K$, the zero-centered ball of radius $r_\Q$ in an RKHS with kernel $\tilde k(\cdot, \cdot).$ We take both $k$ and $\tilde k$ to be Gaussian RBF kernel. 
The bandwidth of $k$ and $\tilde k$ are selected to make sure the function Bellman loss is not large on a validation set.
The radius is selected to be sufficiently large to ensure that $q\true$ is included in $\Q$. 
To ensure a sufficiently large radius, we use the data to approximate a $\hat{q}$ so that its functional Bellman loss is small than $\epsilon_n$.
Then we set $r_{\Q} = 10 * \|\hat{q}\|_{\tilde \K}$.
We optimize $\ratio$ using the random feature approximation method described in Appendix~\ref{sec:kernelw}. 
Once $\ratio_+$ and $\ratio_-$ are found, we evaluate the bound in Eq.~\eqref{equ:mainVbounddelta} exactly, to ensure the theoretical guarantee holds. 

\subsection{{Sensitivity to Hyper-Parameters}}\label{sec:sense_hypers}
We investigate the sensitivity of our algorithm to the choice of hyper-parameters.
The hyper-parameter mainly depends on how we choose our function class $\Q$ and $\W$.
\newcommand{\absize}{.33\textwidth}
\begin{figure}[t]
    \centering
    \begin{tabular}{ccc}
        \raisebox{2.5em}{\rotatebox{90}{\small $\log_{10}(\text{interval length})$}} 
        \includegraphics[width = \absize]{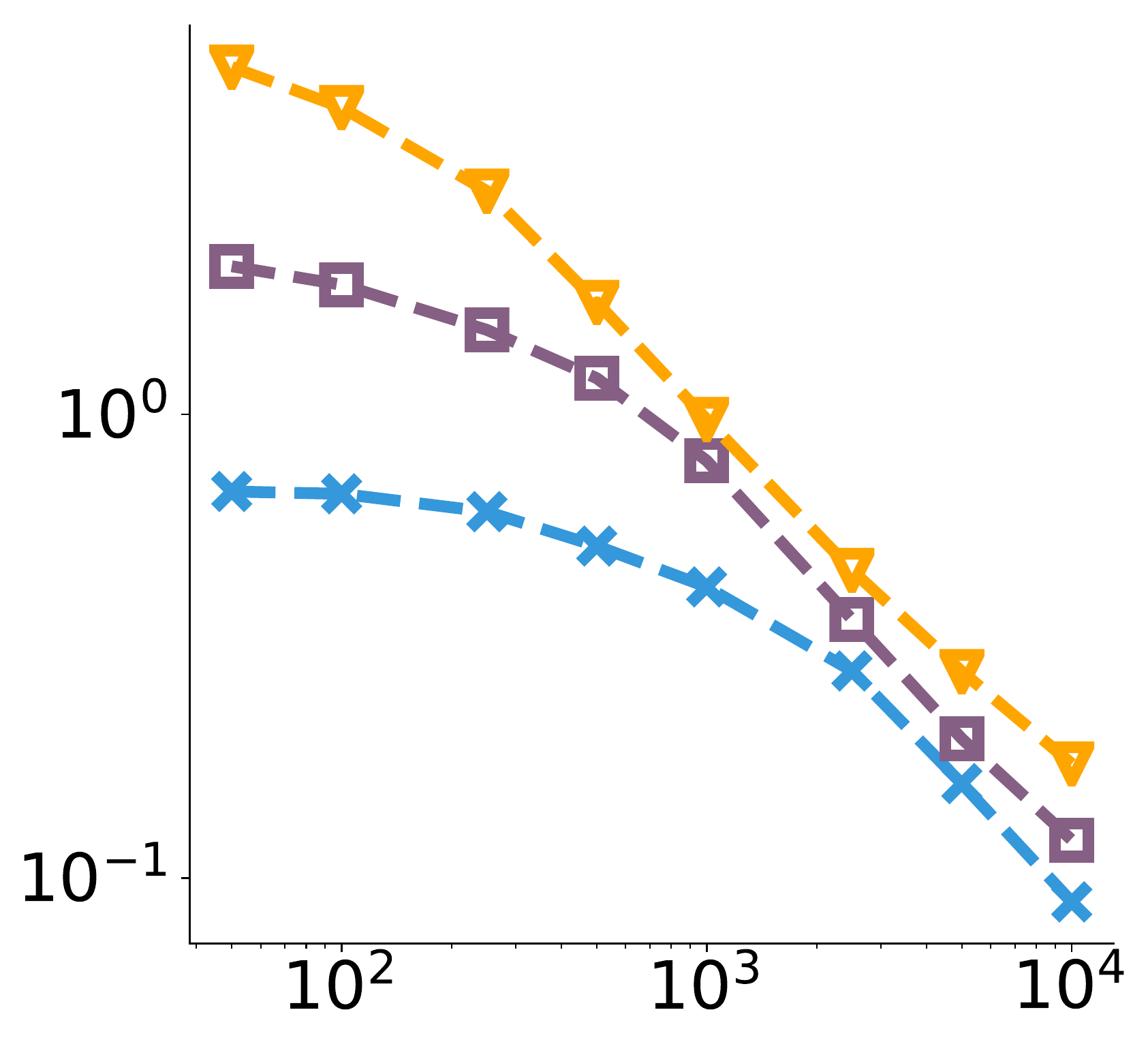} & 
        \raisebox{5.0em}{\rotatebox{90}{\small Reward}} 
        \includegraphics[width = \absize]{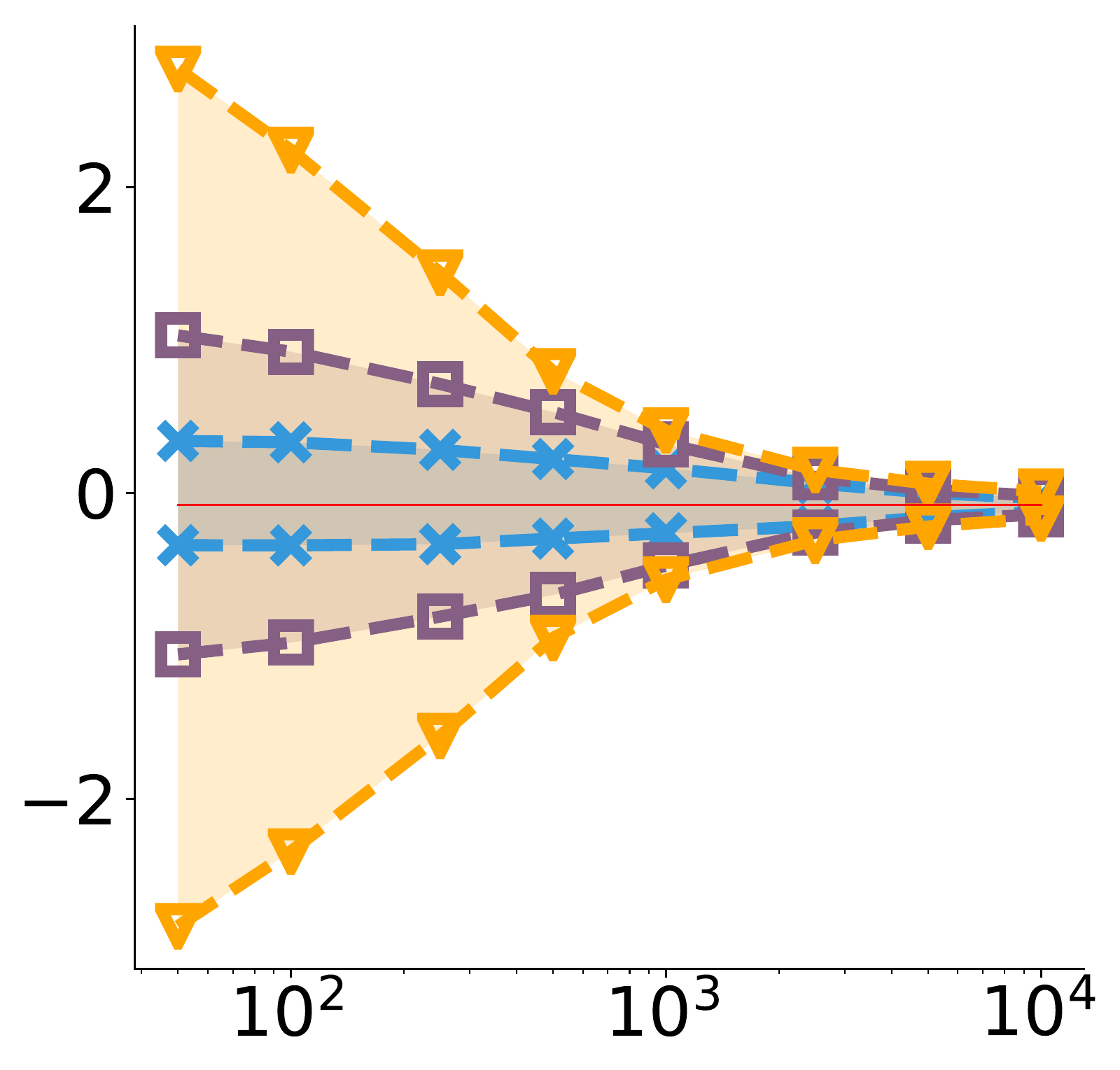} &
        \raisebox{3.5em}{\includegraphics[width = .19\textwidth]{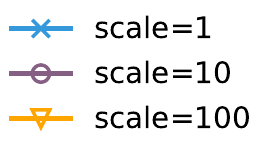}}\\
        \small{(a)~Number of transitions,~$n$} & \small{(b)~Number of transitions ~$n$} &\\
    \end{tabular}
    \caption{Ablation study on the radius $r_{\Q}$ of the function class $\Q$. \zt{The default collecting procedure uses a horizon length of $H = 50$. The discounted factor is $\gamma = 0.95$ by default.}}
    \label{fig:ablation_radius}
\end{figure}

\myparagraph{Radius of $\Q$}
Recall that 
we choose $\Q$ to be a ball in RKHS with radius $r_{\Q}$, that is, %
$$
\Q = r_{\Q}\tilde \K = \{r_{\Q} f: f\in \tilde \K\},
$$
where $\tilde \K$ is the unit ball of the RKHS with kernel $\tilde k$. 
Ideally, we want to ensure that $r_{\Q} \geq \|q\true\|_{\tilde \K}$ so that $q\true \in \Q$.

Since it is hard to analyze the behavior of the algorithm when $q\true$ is unknown,
we consider a synthetic environment where the true $q\true$ is known.
This is done by explicitly specifying a $q\true$ inside $\tilde \K$ 
and then infer 
the corresponding deterministic reward function $r(x)$ by inverting the Bellman equation: 
$$
r(x) \defeq q\true(x) - \gamma \E_{x'\sim \dist P_\pi(\cdot|x)}[q\true(x')].
$$
Here $r$ is a deterministic function, instead of a random variable, with an abuse of notation.  
In this way, we can get access to the true RKHS norm of $q\true$:
$$
\rho^* = \|q\true\|_{\tilde \K} \,.
$$
For simplicity, we set both the state space $\set S$ and action space $\set A$ to be $\RR$ and set a Gaussian policy $\pi(a|s)\propto \exp(f(s,a)/\tau)$, where $\tau$ is a positive temperature parameter. We set $\tau = 0.1$ as target policy and $\tau = 1$ as behavior policy.

Figure \ref{fig:ablation_radius} shows the results as we set 
 $r_{\Q}$ to be $\rho^*$, $10\rho^*$ and $100\rho^*$, respectively.
We can see that the tightness of the bound is affected significantly by the radius when the number $n$ of samples is very small. 
However, as the number $n$ of samples grow (e.g., $n\geq 2\times 10^3$ in our experiment), the length of the bounds become less sensitive to the changing of the predefined norm of $\Q$.

\begin{figure}[t]
    \centering
    \begin{tabular}{cc}
        \raisebox{4.5em}{\rotatebox{90}{\small Reward}} \includegraphics[width=0.33\textwidth]{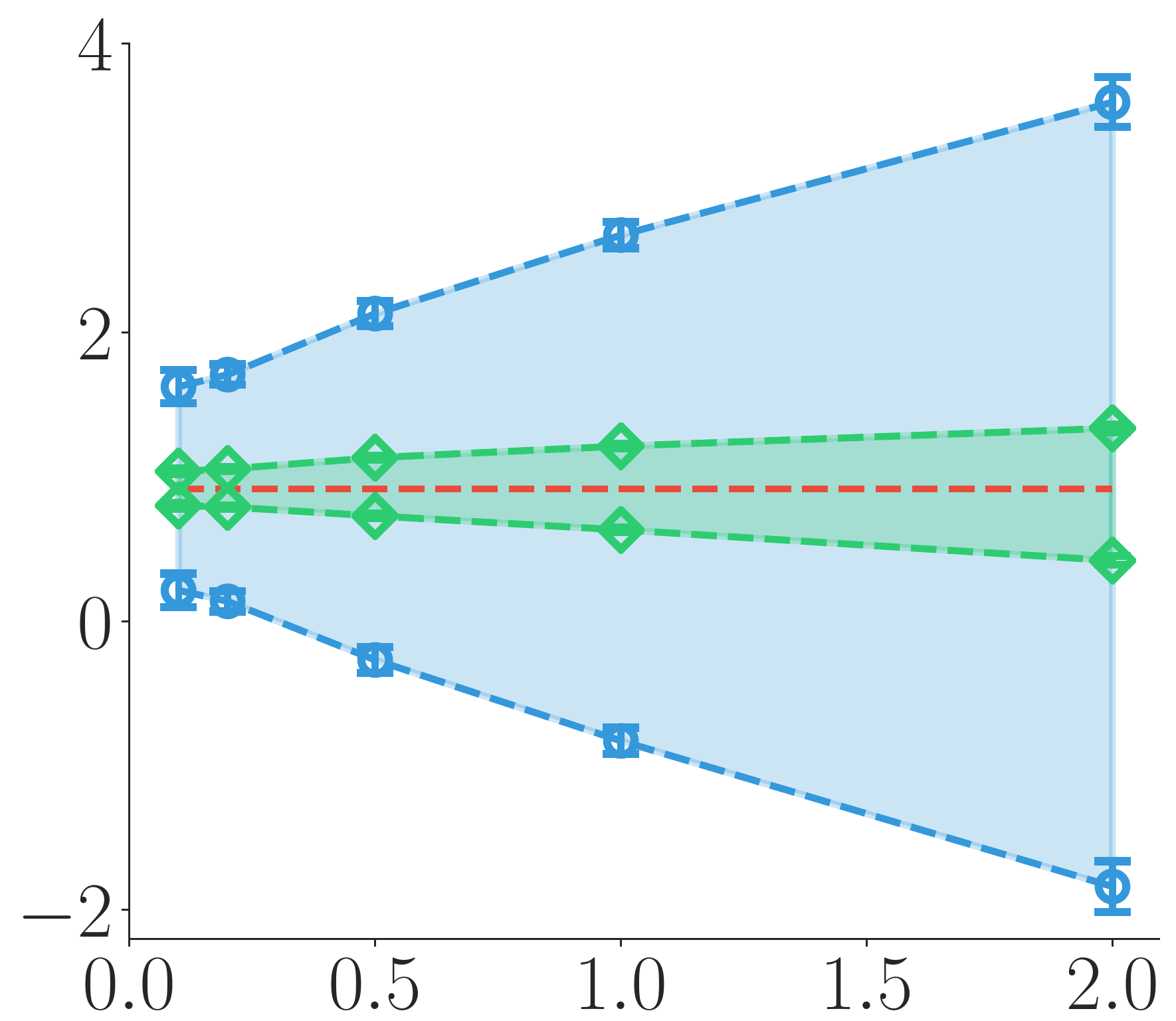} & 
        \raisebox{4.5em}{\rotatebox{90}{\small Interval length}}\includegraphics[width=0.52\textwidth]{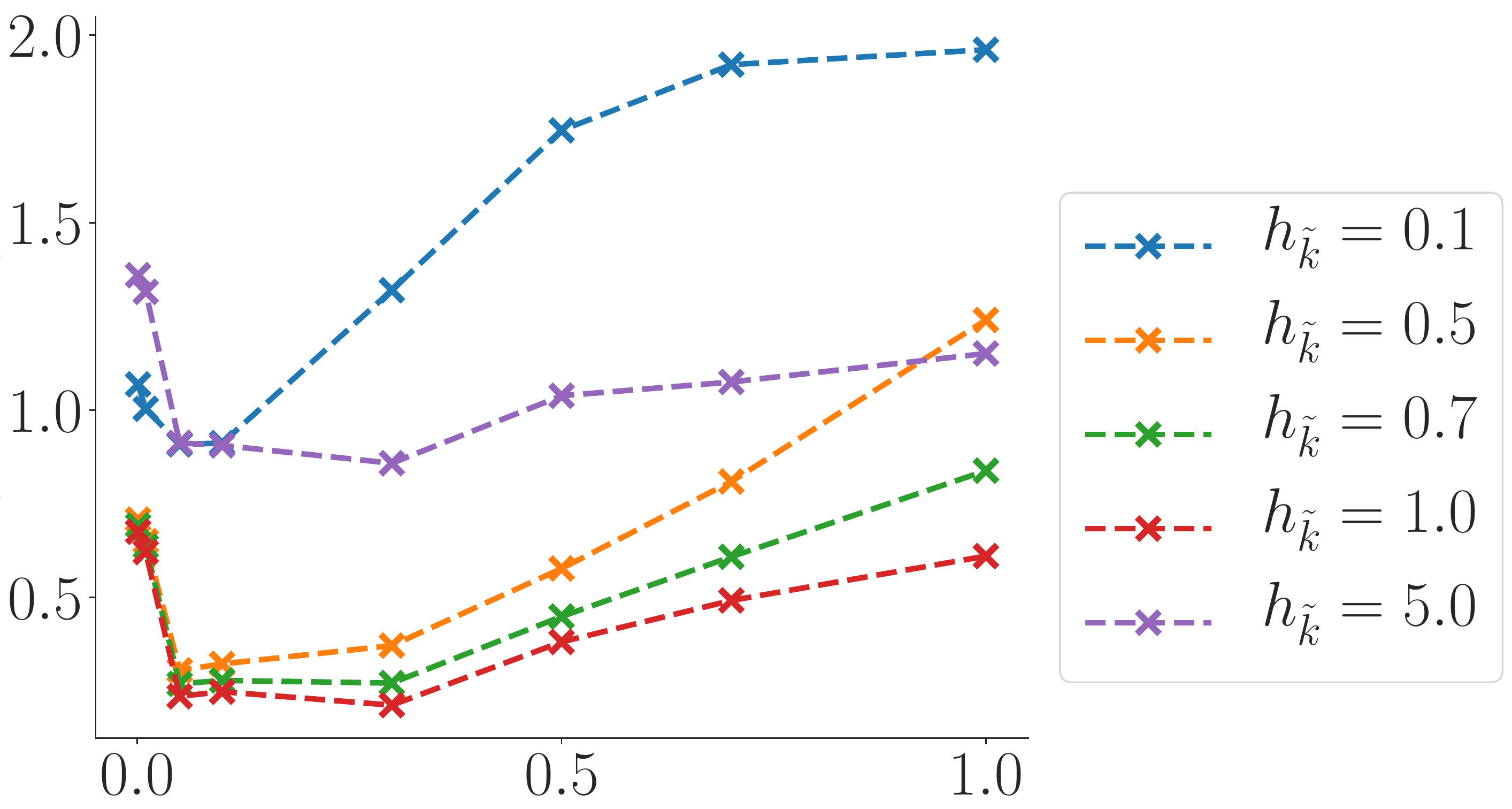}\vspace{-.5em} \\
        \hspace{2em}\footnotesize{Temperature $\tau$} & \hspace{-5em}\footnotesize{Bandwidth of $k$ of $\H$} \\
        \small{(a) Varying temperature $\tau$.}%
        & \small{(b) Varying the bandwidth of kernels in $\H$ and $\set Q$.}
        \vspace{-.5em}
    \end{tabular}
    \caption{\small{Ablation studies on Inverted-Pendulum. We change the temperature $\tau$ of the behavior policies in (a),  and change the bandwidth of the kernel  $ k$ of $\H$ and the kernel $\tilde k$ of $\set Q$ (denoted by $h_{\tilde{k}}$ in (b)). 
    }
    }
    \label{fig:ablation_pendulum}
\end{figure}

\myparagraph{Similarity Between Behavior Policy and Target Policy}
We study the performance of changing temperature of the behavior policy. 
We test on Inverted-Pendulum environment as previous described.
Not surprisingly, we can see that the  closer the behavior policy to the target policy (with temperature $\tau = 0.1$), the tighter our confidence interval will be, which is observed in  Figure~\ref{fig:ablation_pendulum}(a).

\myparagraph{Bandwidth of RBF kernels}
We study the results as we change the bandwidth in kernel $k$ and $\tilde k$ for $\W$ and $\Q$, respectively. 
Figure~\ref{fig:ablation_pendulum}(b) shows 
the length of the confidence interval when we use different bandwidth pairs
in the Inverted-Pendulum environment. 
We can see that 
 we get relatively tight confidence bounds 
as long as we set the bandwidth in a reasonable region~(e.g., we set the bandwidth of $k$ in $[0.1, 0.5]$, the bandwidth of $\tilde{k}$ in $[0.5, 3]$). %

\renewcommand{\absize}{.24\textwidth}
\begin{figure}[t]
    \centering
    \begin{tabular}{ccc}
        \hspace{-.5em}
        \raisebox{.8em}{\rotatebox{90}{\footnotesize{$\log_{10}(\text{interval length})$}}} 
        \includegraphics[width = \absize]{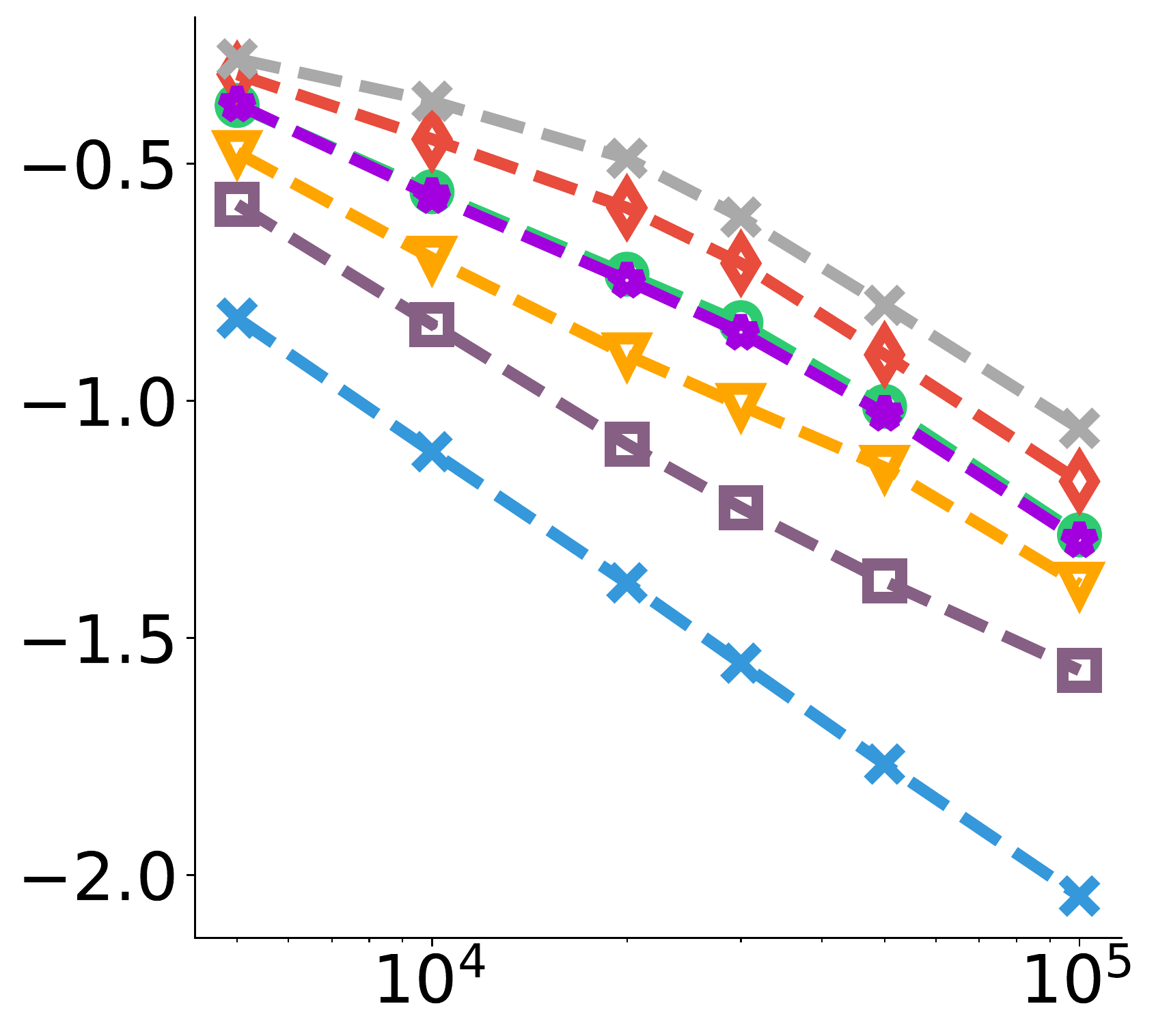} 
        \raisebox{1.0em}{\includegraphics[width = .13\textwidth]{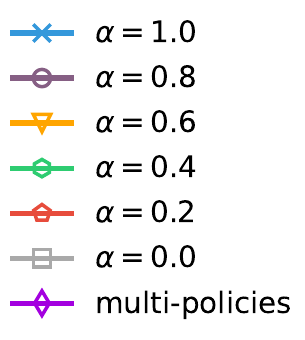}}
        & 
        \hspace{-2.5em}
        \includegraphics[width = \absize]{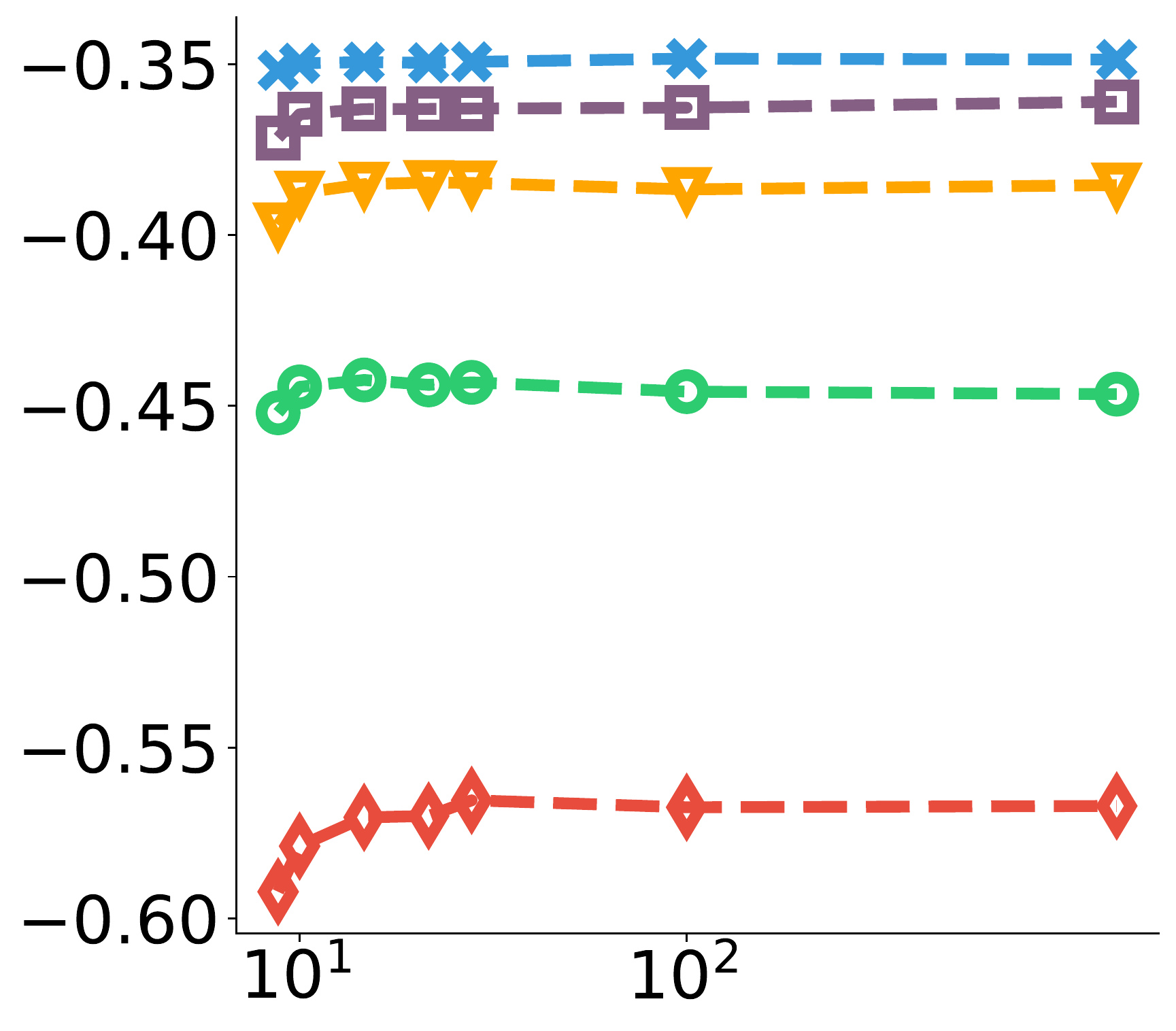}
        &
        \hspace{-1.5em}
        \includegraphics[width = \absize]{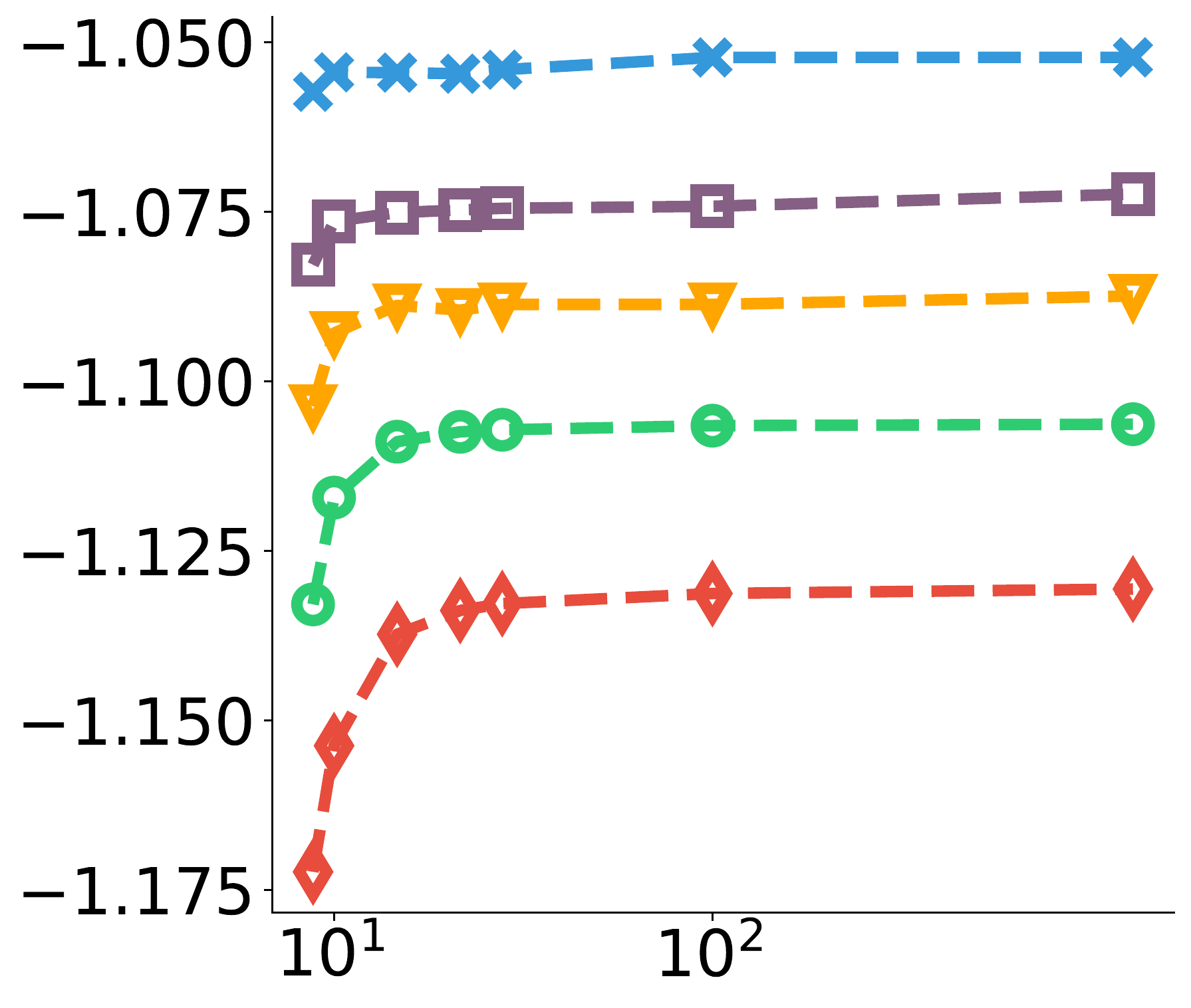}
        \raisebox{1.7em}{\includegraphics[width = .10\textwidth]{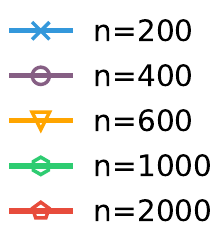}}
        \vspace{-0.3em}\\
        \hspace{-5.5em}\scriptsize{ number of transitions, $n$} 
        & \hspace{-3.5em}\scriptsize{Trajectory length $T$} & \hspace{-2.5em}\scriptsize{Trajectory length $T$}\\
        \hspace{-5.5em}\small{ (a) Varying $n$, $\alpha$, fixed $\gamma = 0.95$ }
        & \hspace{-3.5em}\small{(b) Varying $n$, $T$, fixed $\gamma = 0.95$}
        & \hspace{-2.5em}\small{ (c)Varying $n$, $\alpha$, fixed $\gamma = 0.99$}
    \end{tabular}
    \caption{\small{Ablation studies on the data collection procedure, as we (a) change the behavior policies, and (b)-(c) change the trajectory lengths. 
    The other settings are the same as that in Figure~\ref{fig:ablation_radius}. 
    }
    }
    \label{fig:ablation_data}
\end{figure}

\subsection{\zt{Sensitivity to the Data Collection Procedure}}\label{sec:h3}
We investigate the sensitivity of our method as we use different behavior policies to collect the dataset $\emp D_n$. 

\paragraph{Varying Behavior Policies}

We study the effect of using different behavior policies. 
We consider the following cases:
\begin{enumerate}
    
\item  Data is collected from a single behavior policy of form $\pi_{\alpha} = \alpha \pi + (1-\alpha)\pi_0$, where $\pi$ is the target policy and $\pi_0$ is another policy. 
We construct $\pi$ and $\pi_0$  to be  Gaussian policies of form $\pi(a|s)\propto \exp(f(s,a)/\tau)$ with different temperature $\tau$, where temperature for target policy is $\tau=0.1$ and temperature for $\pi_0$ is $\tau = 1$. 

\item The dataset $\emp D_n$ is the combination of the data collected from multiple behavior policies of form $\pi_\alpha$ defined as above, with 
$\alpha\in\{{0.0}, {0.2}, {0.4}, {0.6}, {0.8}\}$. %

\end{enumerate}

We show in Figure~\ref{fig:ablation_data}(a) that the length of the confidence intervals by our method 
as we vary the number $n$ of transition pairs and the mixture rate $\alpha$.  
We can see that the length of the interval decays with the sample size $n$ 
for all mixture rate $\alpha$. 
Larger $\alpha$ yields better performance because the behavior policies are  closer to the target policy. 

\paragraph{Varying Trajectory Length $T$ in $\emp D_n$}
As we collect $\emp D_n$, 
we can either have a small number of long trajectories, 
or a larger number of short trajectories. 
In Figure~\ref{fig:ablation_data}(b)-(c), 
we vary the length $T$ of the trajectories as we collect $\emp D_n$, while 
fixing the total number $n$ of transition pairs.
In this way, the number of trajectories in each $\emp D_n$ would be $m = n/T$. 
We can see that the trajectory length does not impact the results significantly, especially when the  length is reasonably large (e.g., $T \geq 20$).

\section{Finite Horzion and Time-varying MDPs}\label{sec:finitehorizon}
In this paper we mainly focus on infinite-horizion and time-homogeneous MDPs.
Here we consider how to 
extend our method to finite-horizon and time-inhomogeneous MDPs, 
where the transition probability and reward are time-dependent and we have a finite %
horizon length $H<+\infty$. 
That is, 
we consider cases when the next state and reward follow an unknown, time-dependent transition distribution, 
$$
(r_t, s_{t+1}) \sim \dist P(\cdot|s_t, a_t;~~t), 
$$
and we want to estimate the following finite-horizon expected reward 
$$
J_{\pi,P, H} = \E_{\pi,\dist P} \left [ \sum_{t=0}^{H-1} \gamma^t r_t~|~s_0\sim \dist D_0\right ], 
$$
where the horizon length $H$ is finite but can be  large. 
We should distinguish $H$ (which is a part of the problem definition) with the trajectory length $T$  of the data in $\emp D_n$ (which is a part of the data collection procedure).

We can approach this problem by transforming it into an infinite-horizon  and time-homogeneous problem and then apply our method. 
This can be done by incoporating the time $t$ as a part of the state vector. 
To be concrete, assume we collect a set of 
 transition pairs $(s_i, a_i, r_i, s_i', t_i)_{i=1}^n$, where in addition to $(s_i, a_i, r_i, s_i')$, we also record the time $t_i$ when the transition occurs. 
We define an augmented state 
$\bar s_i = [s_i, t_i]$ and $\bar s_{i}' = [s_i, t_i + 1]$ 
which include 
the time $t_i$ as a part. %
Our method can be then applied without modification on the augmented dataset $\{\bar s_i, a_i, r_i\}_{i=1}^n$, once we  ensure that $q([s,t],a) = 0$ for all $t\geq H$ when defining the function set $\set Q$.  
To build an RKHS $\set Q$ that satisfy $q([s,t],a) = 0,~\forall t\geq H,~q \in \set Q$, we just need to use a  kernel $\tilde k([s,a, t], [s',a',t'])$ such that $\tilde k([s,a, t], [s',a',t'])=0$ whenever $t\geq H$ or $t' \geq H$.

\end{document}